\documentclass[11pt]{article}

\usepackage{amsthm}
\usepackage{braket}
\usepackage{diagbox}
\usepackage{float}
\usepackage{enumitem}
\usepackage{tikz-cd}
\usepackage{makecell}
\usepackage{siunitx}
\usepackage{wrapfig}
\usepackage{bbm}

\usepackage[margin=1in]{geometry}
\usepackage[T1]{fontenc}
\usepackage{microtype}

\usepackage{amsmath,amssymb,amsfonts}
\usepackage{newtxtext,newtxmath}
\usepackage{graphicx}
\usepackage{booktabs}
\usepackage{caption}
\usepackage{subcaption}

\usepackage[numbers,sort&compress]{natbib}

\usepackage{xcolor}
\usepackage[hidelinks]{hyperref}
\usepackage[nameinlink,noabbrev]{cleveref}

\usepackage{authblk}

\usepackage{titlesec}
\titlespacing*{\section}{0pt}{2.0ex plus .5ex minus .2ex}{1.0ex}
\titlespacing*{\subsection}{0pt}{1.5ex plus .4ex minus .2ex}{0.8ex}

\usepackage{xspace}
\makeatletter
\providecommand{\bmvaOneDot}{\futurelet\@let@token\@bmvaonedotaux}
\def\@bmvaonedotaux{\ifx\@let@token.\else.\null\fi\xspace}
\providecommand{\bmvaUrl}[1]{\url{#1}}
\providecommand{\runninghead}[2]{}
\providecommand{\bmvcreviewcopy}[1]{}
\makeatother

\newtheorem{definition}{Definition}
\newtheorem{proposition}{Proposition}
\newtheorem{theorem}{Theorem}
\newtheorem*{conjecture}{Conjecture}
\newtheorem{lemma}{Lemma}

\DeclareUnicodeCharacter{1F99C}{}
\DeclareUnicodeCharacter{2217}{*}

\newcommand{\rmA}{\ensuremath{\mathrm{A}}}
\newcommand{\rmB}{\ensuremath{\mathrm{B}}}
\newcommand{\rmE}{\ensuremath{\mathrm{E}}}

\definecolor{codeslate}{rgb}{0.30, 0.35, 0.40}
\definecolor{codenavy}{rgb}{0.15, 0.25, 0.45}

\NewDocumentEnvironment{eqcode}{ m +b }
{%
    \par\vspace{\abovedisplayskip}\noindent
    \refstepcounter{equation}%
    \hfill
    $\displaystyle #2 \qquad \textup{(\theequation)}$%
    \hfill 
    \textcolor{gray}{\scriptsize\texttt{#1}}%
    \hspace{2em}
    \par\vspace{\belowdisplayskip}%
}{}

\title{\vspace{-1.5em}A Unified Framework for Vision Transformers Equivariant to Discrete Subgroups of $\mathrm{O}(2)$}

\author[1]{T\=\i kun \^Ong}
\author[2]{Georg B\"okman}
\affil[1]{Independent Researcher}
\affil[2]{University of Amsterdam}

\date{\vspace{-1.0em}}

\begin{document}

\maketitle

\begin{abstract}
	Vision transformers have become a dominant architecture for visual recognition.
	However, standard models do not explicitly encode the planar symmetries that
	arise in many vision domains. We introduce a family of vision transformers
	equivariant to arbitrary discrete subgroups of $\mathrm{O}(2)$, providing a
	unified framework that generalizes prior flipping- and $D_4$-equivariant
	transformer architectures. Our construction yields equivariant analogues of the
	core transformer components, together with expressivity guarantees for the
	resulting layers. In particular, we show that whenever $H \le G$, the class of
	$G$-equivariant ViTs embeds naturally into the class of $H$-equivariant ViTs. We
	also prove that, in the single-head setting, the corresponding equivariant
	self-attention layer realizes every $G$-equivariant self-attention map
	representable by ordinary self-attention.
	We further construct a $D_6$-equivariant model based on hexagonal patches,
	making the architecture compatible with six-fold rotational symmetries. We
	evaluate the resulting models on the PatternNet aerial image dataset in
	artificially data-scarce regimes across subgroups of $D_4$ and $D_6$. Our
	experiments compare two equivariant attention mechanisms and analyze how the
	choice of homogeneous-space configurations used in the nonlinearities affects
	performance. Preliminary results under matched parameter budgets indicate that
	equivariance can improve recognition accuracy, motivating further study of how
	discrete symmetry groups shape transformer-based visual recognition models.
\end{abstract}


\section{Introduction}
\textit{Geometric deep learning} is concerned with designing model
architectures that systematically incorporate symmetry and geometry of a
learning task as inductive bias
\cite{cohenGeneralTheoryEquivariant2020,bronsteinGeometricDeepLearning2021,gerkenGeometricDeepLearning2023}.
An important class of such models are the so-called \textit{group-equivariant}
neural networks
\cite{shawe-taylorBuildingSymmetriesFeedforward1989, cohenGroupEquivariantConvolutional2016},
which enjoy layer-wise group equivariance. That is, the map represented by each
layer, mapping input features to output features, commutes with a priori specified group actions.
Using a group-equivariant neural network exploits the existence of a group action on the data, for instance translations and rotations of images, and aim to simplify the learning task by hard-coding this symmetry into the neural network architecture. 
An important special case of equivariance is
\textit{invariance}, where the group action on the output of the network is
trivial.
Image classification of aerial imagery is a prototypical invariant task, which we will consider in the experiments.

So far, apart from graph neural networks with equivariant features
\cite{andersonCormorantCovariantMolecular2019,maronInvariantEquivariantGraph2019,satorrasEquivariantGraphNeural2022,eijkelboomEquivariantMessagePassing2023,batatiaMACEHigherOrder2023, batznerE3equivariantGraphNeural2022},
which have been popular in particular due to their applications in chemistry \cite{zhangArtificialIntelligenceScience2025},
the most prominent examples of equivariant or invariant neural networks
are generalizations of convolutional neural networks (CNNs)
\cite{cohenGroupEquivariantConvolutional2016,cohenSteerableCNNs2016,weiler3DSteerableCNNs2018,
	kondorGeneralizationEquivarianceConvolution2018,kondorClebschGordanNetsFully2018,veelingRotationEquivariantCNNs2018,hoogeboomHexaConv2018,
	weilerGeneral$E2$EquivariantSteerable2021,bekkersRotoTranslationCovariantConvolutional2018,heEfficientEquivariantNetwork2021,
	kunduGeometricApproachSteerable2025}.
These often involve convolutions over groups, with the usual convolution being the special case for
the group $\mathbb{R}^n$ or $\mathbb{Z}^n$, equivariant to translations.
CNNs equivariant to discrete subgroups of the roto-reflection group $\mathrm{O}(2)$ have been widely studied and often outperform ordinary CNNs in the equal  parameter setting \cite{weilerGeneral$E2$EquivariantSteerable2021}.

Recently, CNNs have been replaced by vision transformers (ViTs) in many
state-of-the-art computer vision models \cite{dosovitskiyImageWorth16x162021,
oquabDINOv2LearningRobust2024, Wang_2025_CVPR}. There are multiple reasons for
preferring ViTs, including ease of capturing long-range relationships between
different parts of an image (or multiple images) and architectural alignment
with networks used for other data modalities, such as large language models.
Given the success of equivariant CNNs, it is natural to consider
equivariant ViTs, which are the main objects of study in this paper.

We take a representation-theoretic view of equivariant vision transformers for
discrete subgroups $G \le \mathrm{O}(2)$. This viewpoint subsumes the flipping-
and $D_4$-equivariant ViTs of Refs.~\cite{bokmanFloppingFLOPsLeveraging2025,
nordstromOcticVisionTransformers2025}, while making it possible to reason
about the resulting model classes independently of any particular symmetry
group. An important part of our analysis is to compare these classes as the
symmetry group varies. 
We show, for instance, that imposing a larger symmetry group does not lead to an
unrelated architecture: a $G$-equivariant ViT can be regarded naturally as an
$H$-equivariant ViT for every subgroup $H \le G$. We also analyze the
self-attention layer itself and prove that, at least for a single attention
head, the equivariant parameterization loses no expressive power relative to
standard self-attention once one restricts to maps that are $G$-equivariant. At
the same time, this formalism brings several architectural choices into focus,
including nonlinearities constructed from arbitrary $G$-sets and several possibilities for
equivariant self-attention mechanisms. 
The experiments in Section~\ref{sec:experiments} are designed to probe these choices in
controlled, small-scale settings, rather than to optimize for large-scale
benchmark performance.

%


By separating the representation-theoretic structure from group-specific implementation choices, the framework provides a common basis for constructing, comparing, and analyzing equivariant ViTs across different planar symmetry groups.
\section{Related Works}

Our work is a generalization of equivariant transformers presented in
Refs.~\cite{bokmanFloppingFLOPsLeveraging2025,nordstromOcticVisionTransformers2025},
rendering these architectures as special cases for $G=D_1$ (mirror symmetry) and $G=D_4$.
These architectures in turn closely follow the original Vision Transformer
\cite{dosovitskiyImageWorth16x162021}, which can be seen as the $G=\{e\}$
(trivial group) case.

Other types of equivariant vision transformers have been considered in the literature.
Most notably, Ref.~\cite{xu$E2$EquivariantVisionTransformer2023} uses a
``lifting self-attention'' layer in the very beginning to lift token features
to functions on the group $G$ (i.e., spatial domain features).
In Ref.\ \cite{islamPlatonicTransformersSolid2025}, discrete subgroups $G$ of
$\mathrm{O}(2)$ (as well as $\mathrm{O}(3)$) are considered, where spatial domain features are used
in conjunction with group convolutions for equivariant linear layers.
There are also several works on equivariant transformers for point cloud data. 
In
Ref.~\cite{assaadVNTransformerRotationEquivariantAttention2023}, where an
$\mathrm{SO}(3)$-equivariant attention mechanism for 3D point clouds is
presented, the token feature vectors transform in the defining
(fundamental/three-dimensional) representation of $\mathrm{SO}(3)$.
In Refs.~\cite{fuchsSE3Transformers3DRotoTranslation2020,chatzipantazisSE3EquivariantAttentionNetworks2023},
higher-order $\mathrm{SO}(3)$-tensors (higher-dimensional irreps)
are also used.

We would also like to note that there is another line of work on equivariant
architectures which aims to achieve equivariance by having the model learn to
rotate an input image to its ``canonical'' orientation
\cite{taiEquivariantTransformerNetworks2019,jaderbergSpatialTransformerNetworks2016, kabaEquivarianceLearnedCanonicalization2023}.
These models are frequently called (spatial) transformers in the literature, but their approach is completely distinct to what is commonly referred to as a Transformer following the landmark \textit{Attention Is All You Need} paper, Ref.~\cite{vaswaniAttentionAllYou2017}.
Our vision transformers are transformers in the sense of Ref.~\cite{vaswaniAttentionAllYou2017}.

\section{Group theory preliminaries}
\label{sec:group_theory}

We will assume some familiarity with group theory, a good reference is Serre's textbook on representation theory, Ref. \cite{serreLinearRepresentations}. Let $G$ be a finite
group acting on two sets $X,Y$. A map $\phi: X\rightarrow Y$ is said to be
\textit{$G$-equivariant} if $\phi$ commutes with the $G$-action. The main goal
of Section~\ref{sec:method} is to construct Vision Transformer layers that are
$G$-equivariant, where $G$ is a discrete subgroup of $\mathrm{O}(2)$.

An important form of group actions are group representations, which are linear group actions on vector spaces.
Since representation theory will play a central role in our construction of
equivariant layers, we briefly recall some basic definitions and mathematical
results here. If not otherwise specified, all vector spaces are over
$\mathbb{R}$. 

\begin{definition}
	Let $G$ be a finite group. 
    A \textbf{representation} of $G$ on a real vector space $V$ 
    is a group homomorphism $\rho: G\rightarrow \mathrm{GL}(V)$.
    \begin{enumerate}[label=(\roman*)]
    \item  A representation $(\rho, V)$ is \textbf{irreducible}
    if there exists no nontrivial subspace $U\subset V$ that is
    invariant under the $G$-action.
    \item A representation $(\rho, V)$ on an inner product space $V$
    is \textbf{orthogonal} if $\rho(G) \subset \mathrm{O}(V)$.
    \end{enumerate}
\end{definition}
By standard abuse of notation, we will sometimes write $V$ or $\rho$
instead of $(\rho, V)$ to refer to a group representation.
We often use the shorthand \textit{irrep} to refer to irreducible
representations. 
Every group has a one-dimensioanl irrep, called the \textit{trivial representation} and denoted by $\rho_{\text{triv}}$, by sending all elements to the identity $1\times 1$ matrix.
Two representations are said to be \textit{isomorphic} if there exists a $G$-equivariant
linear bijection between them.
By Maschke's theorem, any representation $(\rho, V)$ of
$G$ is isomorphic to a direct sum of irreps. In practice, the vector space $V$
will always come with a natural inner product, and we always take
representations to be orthogonal, which facilitates the construction of
equivariant self-attention (see Section~\ref{sec:method}).

For two $G$-representations $(\rho, V), (\sigma, W)$, we denote by
$\mathrm{Hom}_G(V, W)$ the vector space of $G$-equivariant linear maps $V\rightarrow
	W$.
We will also write $\mathrm{End}_G(V) := \mathrm{Hom}_G(V,V)$, which is an algebra over $\mathbb{R}$.
If $V$ is any real irrep, by Schur's lemma, $\mathrm{End}_G(V)$ is a division
algebra (all nonzero elements are invertible), so $\mathrm{End}_G(V)\cong
	\mathbb{R}, \mathbb{C}, $ or $\mathbb{H}$ by Frobenius' theorem on division algebras,
where $\mathbb{H}$ is the algebra of quaternions. The real irrep $V$ is
then of \textit{real type, complex type,} and \textit{quaternionic type}
respectively. In this paper, as we consider discrete subgroups $G$ of $\mathrm{O}(2)$, real irreps are either
of real type or of complex type and either one- or two-dimensional.
We provide an overview of the irreps of $G$ in the Supplementary Material.

For our construction of nonlinearities, we will need the following notion:
\begin{definition}
	A \textbf{homogeneous space} of a group $G$ is a set
	$X$ with a transitive $G$-action.
\end{definition}

Here, transitive means that each element $x\in X$ can be taken to any other
element $y\in X$ by the $G$-action. For any $x\in X$, we denote by
$\mathrm{Stab}_G(x):= \{g\in G \mid gx = x\}$ the stabilizer subgroup. It is
then straightforward to show that $X\cong G/\mathrm{Stab}_G(x)$ as
$G$-sets (i.e., there is a $G$-equivariant bijection). Conversely, for
any subgroup $H\le G$, the coset space $G/H$ is naturally a homogeneous space.
For any two subgroups $H,H'\le G$, the homogeneous spaces $G/H$ and
$G/H'$ are isomorphic as $G$-spaces if and only if $H$ and $H'$ are conjugate to each other.
Thus, $G/H$, with one subgroup $H$ from each subgroup conjugacy class,
exhaust all possible homogeneous spaces of $G$ up to isomorphism.

Finally, we would like to fix the notation for a construction that is ubiquitous
in our work and in geometric deep learning in general. For a set $X$ and a
vector space $V$, we denote
by $C(X, V)$ the vector space of all maps $X\rightarrow V$. If there is a
$G$-action on $X$, the vector space $C(X, V)$ is in addition a
$G$-representation, with a group element $g\in G$ acting on a function $f:
X\rightarrow V$ by $(g\cdot f)(x) = f(g^{-1}x)$. 
Clearly, $C(X, V) \cong \mathbb{R}^X\otimes V$ canonically.
If $V$ also carries a $G$-representation $\rho$, then a natural
$G$-representation on $C(X,V)$ is given by $g\cdot f(x) = \rho(g) f(g^{-1}x)$.

\section{Method}
\label{sec:method}
We start by setting up the underlying geometric structure on which the equivariant ViT will operate.
Recall that the \textit{Minkowski sum} of two subsets $X,Y$ of a vector space is 
defined as $X+Y := \{x+y | x\in X, y\in Y\}$.
\begin{definition}
\label{defn:patches}
Let $G$ be a discrete subgroup of $\mathrm{O}(2)$.
A $G$-\textbf{patchified grid} is a set $\mathcal{H}_0\subset \mathbb{R}^2$
that can be written as the Minkowski sum of two $G$-stable finite subsets $U, \mathcal{H}\subset \mathbb{R}^2$.
$U$ is called the \textbf{base patch}, and its translates $U_a:= U + a$, where $a\in \mathcal{H}$, are the \textbf{patches} of $\mathcal{H}_0$.
\end{definition}
Note that we do \textit{not} require the patches $U_a$ to be disjoint. Since $U$ and $\mathcal{H}$ are stable under $G$, so is $\mathcal{H}_0$, implying that all three subsets of $\mathbb{R}^2$ are $G$-sets.

For example, for $G=D_4$ acting on $\mathbb{R}^2$ by reflections and $\ang{90}$ rotations, we can take 
\begin{equation}
\label{eq:d4_patches}
\begin{aligned}
    &U = \left\{-\frac{P-1}{2}, -\frac{P-3}{2}, \dotsb, \frac{P-1}{2}\right\}^2, \quad
    &\mathcal{H} = \left\{-\frac{q-1}{2}P, -\frac{q-3}{2}P, \dotsb, \frac{q-1}{2}P\right\}^2.
\end{aligned}
\end{equation}
Then $\mathcal{H}_0 = U + \mathcal{H}$ is a usual square grid with $(qP)^2$
pixels and $q^2$ disjoint patches, with each patch having $P^2$ pixels.

An RGB image is then an element of $C(\mathcal{H}_0, \mathbb{R}^3)$, and each
image patch is an element of $C(U_a, \mathbb{R}^3)\cong C(U, \mathbb{R}^3)$.
Before discussing the details of each equivariant layer, we would like to
clarify the space in which a single token in an intermediate layer of our model
lives. In order for equivariance to make sense at all, a token feature $x$ must
be an element of a space on which $G$ acts. A simple and natural assumption is that $x$
belongs to a finite-dimensional $G$-representation $V$ over $\mathbb{R}$. By
Maschke's theorem, we can take
\begin{equation}
	\label{eq:feature_space}
	V = \bigoplus_{\rho \in \widehat{G}} \mathbb{R}^{C_\rho}\otimes  V_\rho,
\end{equation}
where $\widehat{G}$ denotes the set of (equivalence classes of) irreps of $G$, 
    $V_\rho$ is the irrep space of $\rho$, and $C_\rho$ is the multiplicity of
the irrep.
A $G$-equivariant ViT (without class tokens) of depth $\delta$ is the
composition 
\begin{equation}
\label{eq:equivit-composition}
\begin{aligned}
    C(\mathcal{H}_0, \mathbb{R}^3) 
    \xrightarrow{\substack{\text{patch embed}\\ \text{\&pos. encoding}}} C(\mathcal{H}, V)
    \xrightarrow{\text{Block}_1} C(\mathcal{H}, V)
    \xrightarrow{\text{Block}_2}
    \dotsb
    \xrightarrow{\text{Block}_\delta}
    C(\mathcal{H}, V),
    \end{aligned}
\end{equation}
where each map is $G$-equivariant. Here, $\mathrm{Block}_i$ is the $i$-th
transformer block, which consists of a multi-head self attention layer followed
by a multilayer perceptron (MLP), both with residual connections. In the
rest of this section, we will elaborate the construction of each layer in
Eq.~\eqref{eq:equivit-composition}.

In practice, elements of the vector space $V$ are stored as a tuple of tensors
of shape $(C_\rho, d_\rho)$, where $d_\rho$ is the dimension of the irrep
$\rho$. For example, the dihedral group $D_6$ of order $12$ has $6$ irreps, labeled by
$\rmA_1, \rmA_2, \rmB_1, \rmB_2$ (one-dimensional), and $\rmE_1, \rmE_2$ (two-dimensional).
A token feature is then represented by $x=(x^{\rmA_1}, x^{\rmA_2}, x^{\rmB_1}, x^{\rmB_2}, x^{\rmE_1}, x^{\rmE_2})$,
where $x^{\rm{A1}}$ has shape $(C_{\rm{A1}}, 1)$, $x^{\rm{E_1}}$ has shape
$(C_{\rm{E_1}}, 2)$, etc.
A visualization of feature maps of a $D_6$-equivariant ViT is shown in Figure~\ref{fig:d6_feature_maps}.

In principle, we allow arbitrary choices of irrep multiplicities $C_\rho$.
A common choice involves $C$ copies of the regular representation,
where $C_\rho = Cd_\rho$ for $\rho$ of real type and $C_\rho = Cd_\rho/2$ for $\rho$ of complex type.

\begin{figure}[H]
\centering
\includegraphics[width=.92\textwidth]{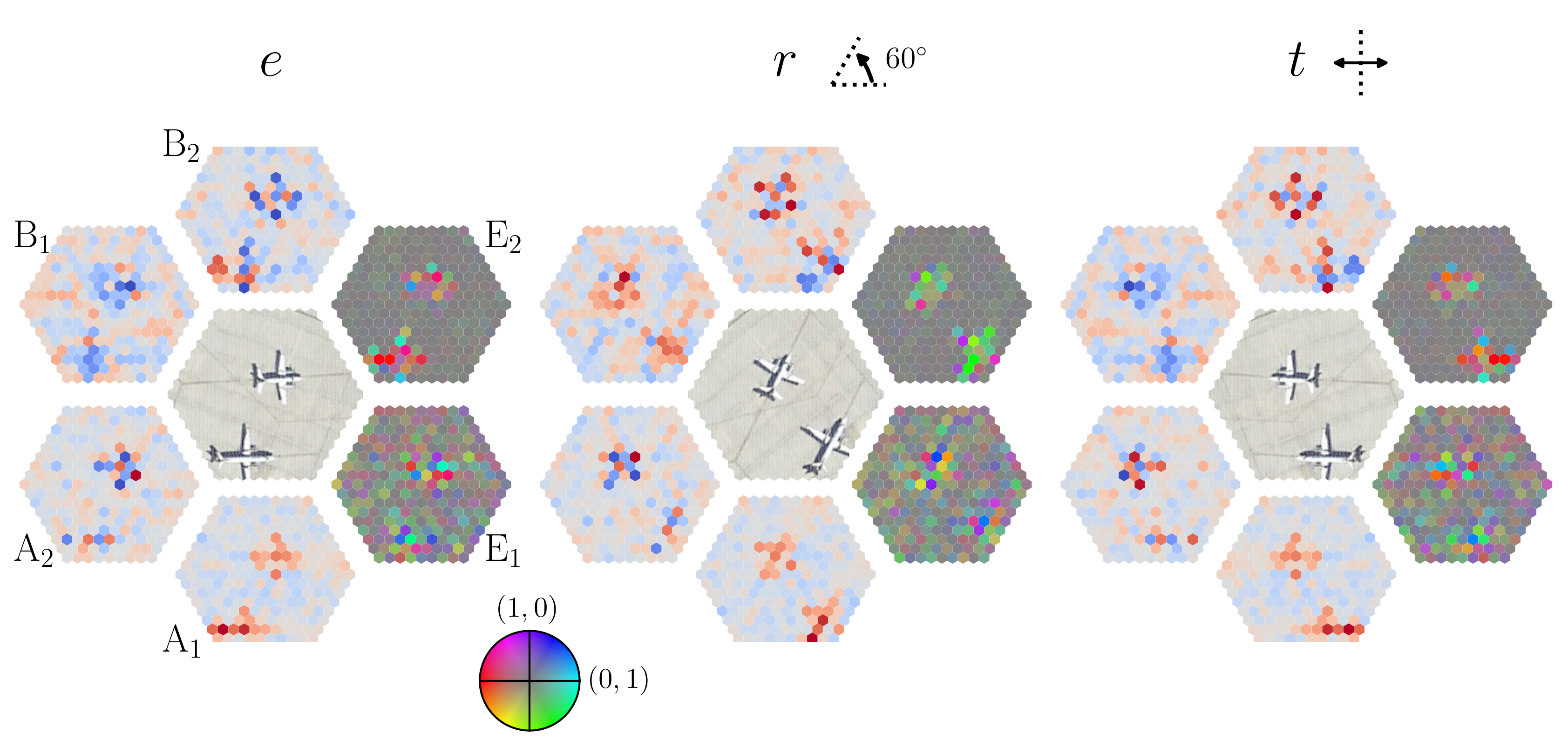}
\caption{Feature maps of a $D_6$-equivariant vision transformer after four transformer blocks in a trained classification model.
For each irrep $\rho \in \widehat{D_6}$, we select a single channel from 
the irrep component $x^\rho \in \mathbb{R}^{\mathcal{H}} \otimes \mathbb{R}^{C_\rho}\otimes V_\rho$.
The features in the two-dimensional irreps $\rmE_1$ and $\rmE_2$ are 
represented by encoding the polar angle and length of a vector in $\mathbb{R}^2$ using the hue and a combination of saturation and brightness respectively (see the color wheel for reference).
For the one-dimensional irreps ($\rmA_1, \rmA_2, \rmB_1, \rmB_2$), red and blue
indicate positive and negative values respectively, with gray representing zero.
}
\label{fig:d6_feature_maps}
\end{figure}

\subsection{Linear Layer}
\label{sec:linear}
We first present a generic equivariant linear layer, which is used in 
our patch embedding layer as well as 
in the MLP and self-attention layers in each transformer block.

Suppose $V' = \bigoplus_{\rho\in\widehat{G}} \mathbb{R}^{C_\rho'}\otimes V_\rho$
and $V'' = \bigoplus_{\rho\in\widehat{G}} \mathbb{R}^{C_\rho''}\otimes V_\rho$
are two $G$-representations decomposed into irreps. Characterizing
$\mathrm{Hom}_G(V', V'')$, the space of $G$-equivariant linear maps
$V'\rightarrow V''$, is straightforward by Schur's lemma:
\begin{equation}
	\mathrm{Hom}_G(V', V'') =
    \bigoplus_{\rho\in\widehat{G}}
    \mathrm{Hom}_G(\mathbb{R}^{C_\rho'}\otimes V_\rho,
    \mathbb{R}^{C_\rho''}\otimes V_\rho)
    =
    \bigoplus_{\rho\in\widehat{G}}
	\mathrm{Mat}_{C_\rho''\times C_\rho'}(\mathbb{R})\otimes 
	\mathrm{End}_G(V_\rho).
\end{equation}
Here, $\mathrm{Mat}_{m\times n}(\mathbb{R})$ denotes the set of $m\times n$ matrices with entries in $\mathbb{R}$.
For practical implementations, this means that a general $G$-equivariant linear
map $W\in \mathrm{Hom}_G(V', V'')$ is an irrep-wise linear map, which we denote by
$W=(W^{(\rho)})_{\rho\in \widehat{G}}$, with
$
	W^{(\rho)} = \sum_{i=1}^{\dim \mathrm{End}_G(V)}w_i \otimes L_i,
$
where $(L_i)_i$ is a chosen basis for $\mathrm{End}_G(V_\rho)$ and $w_i$ contains $C_\rho''\times C_\rho'$ learnable parameters.

If the irrep $V_\rho$ is of real type, then (by definition) $\mathrm{End}_G(V_\rho) \cong
\mathbb{R}$, so a linear map is equivariant if and only if it only acts on the
channel dimension in $\mathbb{R}^{C_\rho'}\otimes V_\rho$.
To implement $W^{(\rho)}$ for a complex-type irrep $V_\rho$, 
which is always two-dimensional in our case (see Section~A of the Supplementary Material),
a convenient choice is 
$L_1=\left(\begin{smallmatrix}1 & 0 \\ 0 & 1\end{smallmatrix}\right)$ and $L_2=\left(\begin{smallmatrix}0 & -1 \\ 1 & 0\end{smallmatrix}\right)$, spanning a subalgebra of $2\times 2$ matrices isomorphic to $\mathbb{C}$.

Finally, we allow the possibility of adding a bias $b\in \mathbb{R}^{C_{\rho_{\text{triv}}}''} \otimes V_{\rho_{\text{triv}}}$ only for the trivial representation.
That is, the complete linear layer is given by
\begin{equation}
	x^{\rho_{\rm{triv}}} \oplus \left(
	\bigoplus_{\rho \in \widehat{G}\setminus \{\rho_{\rm{triv}}\}}
	x^\rho \right)
	\mapsto
	(W^{(\rho_{\rm{triv}})}x^{\rho_{\rm{triv}}} + b) \oplus \left(\bigoplus_{\rho \in \widehat{G}\setminus \{\rho_{\rm{triv}}\}}
	W^{(\rho)}x^\rho\right).
\end{equation}
This exhausts the space of all $G$-equivariant affine maps $V'\rightarrow V''$.
In general, we denote this space by $\mathrm{Aff}_G(V', V'')$.

\subsection{Patch Embedding and Positional Encoding}
\label{sec:pe_pos_enc}
Given a discrete subgroup $G\le \mathrm{O}(2)$
and a $G$-patchified grid $\mathcal{H}_0 = U + \mathcal{H}$
(see Definition~\ref{defn:patches}), 
the patch embedding layer together with added positional encodings
is a $G$-equivariant affine map $C(\mathcal{H}_0, \mathbb{R}^3)\rightarrow C(\mathcal{H}, V)$.
We describe their construction in the following.

\subsubsection{Patch embedding}
For each $a\in \mathcal{H}$, a patch $x(a)$ of an input color image $\mathcal{I}\in C(\mathcal{H}_0, \mathbb{R}^3)\cong \mathbb{R}^{\mathcal{H}_0}\otimes \mathbb{R}^3$ is nothing but the restriction of $\mathcal{I}$ to $U_a$. That is, $x(a) = \mathcal{I}|_{U_a} \in C(U_a, \mathbb{R}^3)$.
Because the $U_a$ are translates of $U$, we can naturally identify $U_a$ and $U$.
The patchification (unfold) map is given by
\begin{equation}
\begin{aligned}
\mathcal{P}: \;C(\mathcal{H}_0, \mathbb{R}^3)
&\rightarrow C(\mathcal{H}, C(U, \mathbb{R}^3))\cong \mathbb{R}^{\mathcal{H}}\otimes \mathbb{R}^U \otimes \mathbb{R}^3, \quad
& \mathcal{I} \mapsto (a\mapsto x(a)).
\end{aligned}
\end{equation}
Note that this map is $G$-equivariant with respect to the natural $G$-action on
both sides.

An equivariant patch embedding layer is then defined as the composition of $\mathcal{P}$
together with an equivariant linear map $l\in \mathrm{Hom}_G(C(U, \mathbb{R}^3), V)$:
\begin{equation}
C(\mathcal{H}_0, \mathbb{R}^3)
\xrightarrow{\mathcal{P}}
C(\mathcal{H}, C(U, \mathbb{R}^3))
\xrightarrow{\mathrm{Id}_{\mathcal{H}}\otimes l}
C(\mathcal{H}, V)
\end{equation}

In practice, one fixes an orthonormal basis $e_1, \dotsb, e_{d_\rho}$ for each irrep $V_\rho$, and
precomputes an orthonormal basis $(E_{\alpha j}^{\rho})_{\rho\in \widehat{G},\alpha\in [\nu_\rho], j\in [d_\rho]}$ for $C(U, \mathbb{R})$, where $\nu_\rho$ is the multiplicity of $\rho$ in $C(U, \mathbb{R})$, such that for each $\rho$ and $\alpha \in [\nu_\rho]$, the linear map
defined by $E^\rho_{\alpha j}\mapsto e_j$ is $G$-equivariant from $\mathrm{span}\{E^\rho_{\alpha 1}, \dotsb, E^\rho_{\alpha d_\rho}\}$ onto $V_\rho$.
Then, the $\rho$-th component of the patch embedding layer is given by
$
y^{\rho}(a)_c = \sum_{c', k,l, \alpha, \mu} K_{\alpha cc'; \mu}^\rho
(L_\mu)_{kl}
\braket{E^{\rho}_{\alpha l}, x(a)_{c'}} e_k,
$
where $K^\rho\in \mathbb{R}^{\nu_\rho\times C_\rho \times 3\times\dim\mathrm{End}_G(V_\rho)}$ is a learnable tensor.
We can intepret $F^\rho_{cc'k}:= \sum_{\alpha,\mu} K^\rho_{\alpha cc';\mu} (L_\mu)_{kl}E^\rho_{\alpha l}\in C(U, \mathbb{R})$ as a filter for the $j$-th component of the irrep $\rho$ for output channel $c$ and input color channel $c'$.
Fig.~\ref{fig:d6-filters} illustrates the patch embedding layer for $G=D_6$, where
both $U$ and $\mathcal{H}$ are taken to be regular hexagons.

\begin{figure}
    \centering
    \includegraphics[width=.9\textwidth]{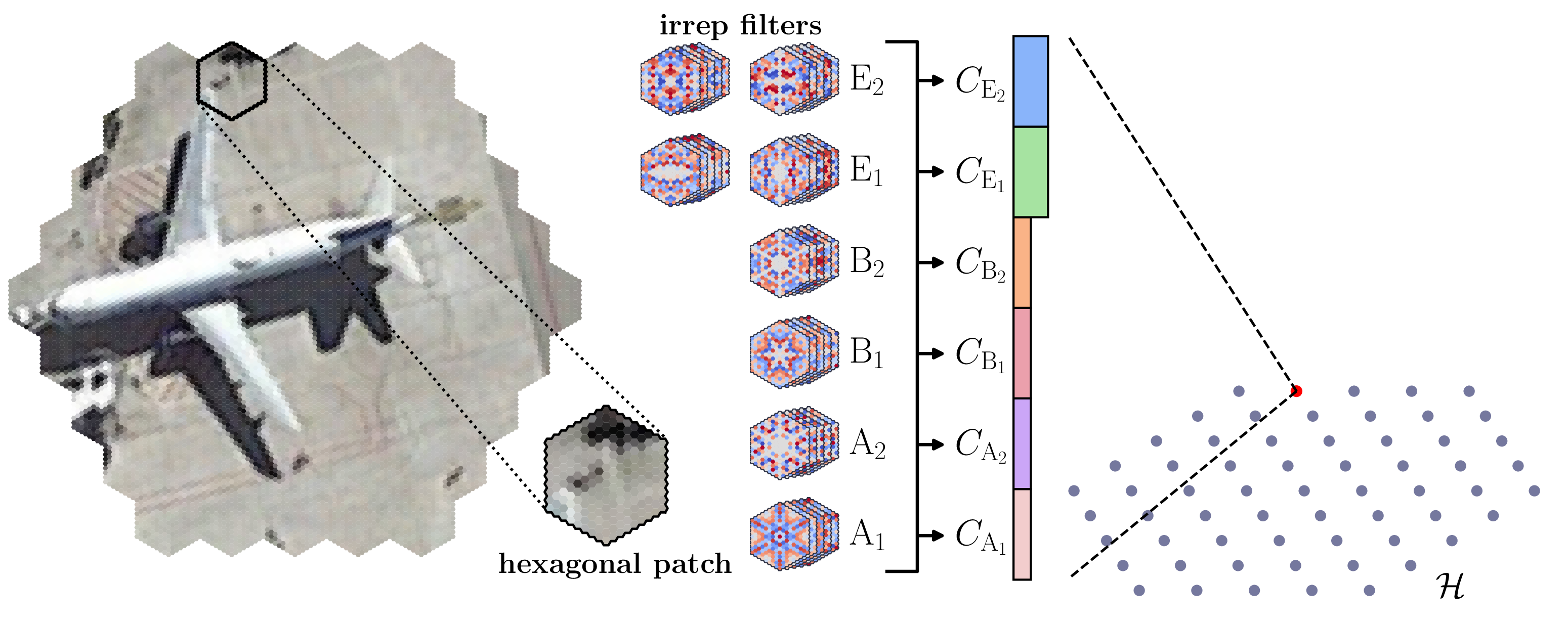}
    \caption{Illustration of the equivariant patch embedding layer with $G=D_6$.}
    \label{fig:d6-filters}
\end{figure}

\subsubsection{Positional encoding}
In this work, we employ learnable absolute positional encodings. That is, a
position-dependent learnable element of $V$ is added to the token features
after the patch embedding layer:
\begin{equation}
    \mathrm{PosEnc}: C(\mathcal{H}, V) \ni x 
    \mapsto x + p \in C(\mathcal{H}, V),
\end{equation}
where $p \in C(\mathcal{H}, V) = \mathbb{R}^{\mathcal{H}}\otimes \left(\bigoplus_{\rho \in \widehat{G}}\mathbb{R}^{C_\rho}\otimes V_\rho\right)$ denotes the positional encodings.
As noted in Ref.~\cite{nordstromOcticVisionTransformers2025}, 
the map $\mathrm{PosEnc}$ is equivariant if and only if $p$ is invariant under the $G$-action.
That is, we require
$
\rho(g) \left[p^{\rho}(g^{-1}\cdot a)\right]
= p^\rho(a).
$
for all irreps $\rho \in \widehat{G}$.
In practice, we precompute a basis for the $G$-invariant subspace 
of $\mathbb{R}^{\mathcal{H}}\otimes V_\rho$, and linearly combine them
using learnable weights during training.

\subsection{Nonlinearity}
\label{sec:nonlin}
While nonlinearity is straightforward to implement if the features
are represented in the ``spatial domain'' of the group,
it is significantly more complex in our case, as our features are
numerically represented as tuples of irrep components.
In this section, we describe a type of nonlinearity that first performs
a Fourier transform (more precisely, a generalization thereof) of
the input features, applies a pointwise nonlinearity, and then transforms back.
As will become clear, not only does this procedure generalize the constructions in
Refs.~\cite{bokmanFloppingFLOPsLeveraging2025,nordstromOcticVisionTransformers2025}
to any finite group, it also allows strictly more freedom in
constructing nonlinearities. In the second part of this section, we argue that 
this is the most general class of equivariant nonlinearities for MLPs under certain natural assumptions.

\subsubsection{The construction}
\label{sec:nonlin-construction}


Fix a (finite) $G$-set $X$. The set $\tilde V:= C(X, \mathbb{R})$ of real-valued
functions is naturally a $G$-representation. If $\sigma: \mathbb{R}\rightarrow \mathbb{R}$ is any function, then the entrywise application of $\sigma$, i.e.,
$C(X, \mathbb{R})\ni (y_m)_{m\in X} \mapsto (\sigma(y_m))_{m\in X}$, is $G$-equivariant.
By Maschke's theorem, the representation $\tilde V$ is isomorphic to a direct
sum of copies of irreps of $G$. Let $\mathrm{FT}$ denote such an isomorphism:
\begin{equation}
	\label{eq:fourier}
	\begin{aligned}
		\mathrm{FT}: \tilde V \xrightarrow{\sim}
		\bigoplus_{\rho\in\widehat{G}} \mathbb{R}^{C'_\rho}\otimes V_\rho
		=: V'
	\end{aligned}
\end{equation}
Here, $\mathrm{FT}$ stands for Fourier transform. The $C'_\rho$ are the
multiplicities of the irreps appearing in $\tilde V$. 
The composition
\begin{equation}
	\label{eq:nonlin-ft}
	V' \xrightarrow{\mathrm{FT}^{-1}} \tilde V \xrightarrow{\text{entrywise\;} \sigma} \tilde V\xrightarrow{\mathrm{FT}}  V'
\end{equation}
is then equivariant and not linear (if $\sigma$ is not linear). 

\begin{wrapfigure}[16]{right}{0.5\textwidth}
\centering
\includegraphics[width=.5\textwidth]{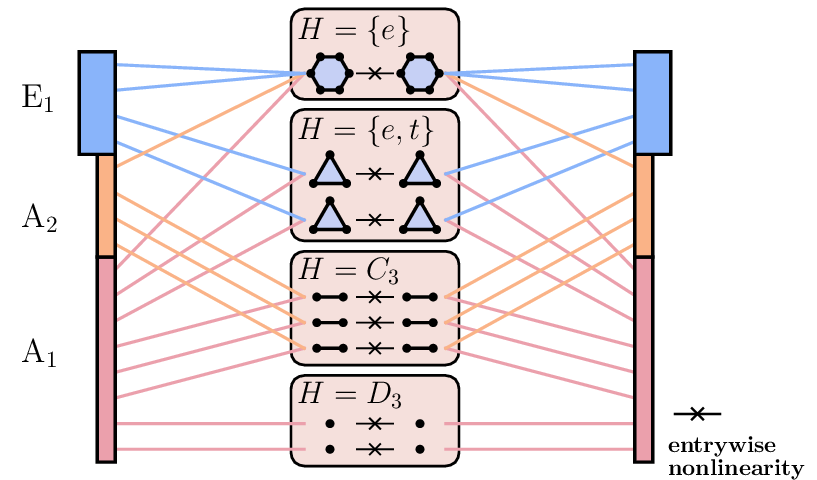}
\caption{Illustration of our equivariant nonlinearity for $G=D_3$.
In this case, $C_{\rmA_1}=8, C_{\rmA_2}=4, C_{\rmE_1} = 4,$
$n_{\{e\}} = 1, n_{\{e,t\}}=2, n_{C_3} = 3,$ and $n_{D_3}=2$.
}
\label{fig:d3-nonlin}
\end{wrapfigure}

For fixed $\sigma$ and up to $G$-equivariant linear bijections $V'\rightarrow
V'$, the map represented by Eq.~\eqref{eq:nonlin-ft} depends only on the orbit
structure of $X$. 
That is, we can decompose $X$ into
a disjoint union of copies of homogeneous $G$-spaces, 
\begin{equation}
\label{eq:homog_spaces}
X=\bigsqcup_{\alpha \in \mathrm{Sub(G)/\sim}}\bigsqcup_{s=1}^{n_\alpha}X_\alpha,
\end{equation}
and the $G$-equivariant MLP constructed using $X$ depends only on the integers
$(n_\alpha)_{\alpha \in \mathrm{Sub}(G)/\sim}$. Here, $\mathrm{Sub}(G)/\sim$ is
the set of equivalence classes of subgroups of $G$ with respect to conjugation,
and $X_\alpha$ is the homogeneous space obtained by taking the quotient 
by any subgroup $H\in \alpha$ in the equivalence class.
Hence, we can think of the $X_\alpha$ as the ``elemenatry lego blocks'' for
constructing the equivariant nonlinear layer. The irrep multiplicities in Eq.~\eqref{eq:fourier} can be related to the $n_\alpha$ by
$C_\rho'= \sum_\alpha \Gamma^\alpha_\rho n_\alpha$, where $\Gamma^\alpha_\rho$
is the multiplicity of irrep $\rho$ in $C(X_\alpha, \mathbb{R})$.

For example, take $G$ to be the dihedral group $D_3$ of order $6$.
It has, in total, $4$ subgroups up to conjugacy. These are given by
\begin{equation}
	\begin{aligned}
		  \textrm{cyclic:}\;\{e\}, \{e,r,r^2\}=C_3
		  \hspace{2em} \textrm{dihedral:}\;\{e,t\}, D_3
	\end{aligned}
\end{equation}
The homogeneous space $D_3 / \{e\}$ is simply the regular group
action (this is true for any group), which decomposes as $\rmA_1 \oplus \rmA_2 \oplus 2\cdot \rmE_1$.
The homogeneous space $D_3/\braket{t}$ is the action on the
three vertices of the triangle. 
Thus, $C(D_3/\braket{t}, \mathbb{R})$ is
three-dimensional, and it is an easy exercise to verify that this representation decomposes as
$\rmA_1\oplus \rmE_1$.
If we take one copy of $D_4/\{e\}$ and two copies of $D_4/\braket{tr}$
in Eq.~\eqref{eq:homog_spaces}, we get in the right hand side of Eq.~\eqref{eq:fourier}
\begin{equation}
	V' =
	(\mathbb{R}^2 \otimes V_{\rmA_1})
	\oplus
	(\mathbb{R}^1 \otimes V_{\rmA_2})
	\oplus
	(\mathbb{R}^3 \otimes V_{\rmE_1}),
\end{equation}
and any nonlinear function $\sigma$ gives rise to an equivariant nonlinearity
$V'\rightarrow V'$ via Eq.~\eqref{eq:nonlin-ft}. See Fig.~\ref{fig:d3-nonlin} 
for an illustration for a more general choice of homogeneous spaces for $D_3$.


\subsubsection{How general is this nonlinearity?}
We will now argue that the construction presented in Section~\ref{sec:nonlin} is
the most general type of nonlinearity for an MLP that is $G$-equivariant, given
some natural assumptions on the nature of the nonlinearity.

We assume that an equivariant MLP layer takes the form $l_2 \circ f \circ l_1$,
where $l_1: V\rightarrow \mathbb{R}^n$ and $l_2: \mathbb{R}^n \rightarrow V$ are
$G$-equivariant affine maps, $\mathbb{R}^n$ carries a $G$-representation, and
$f:\mathbb{R}^n\rightarrow \mathbb{R}^n$ is the entrywise application of any
activation function $\sigma: \mathbb{R}\rightarrow \mathbb{R}$.
The following lemma then implies that this class of MLPs 
coincides with the class of functions representable 
by MLPs with $G$-equivariant affine maps together with nonlinearity constructed according to Eq.~\eqref{eq:nonlin-ft}:
\begin{lemma}
\label{lem:nonlin-permute}
Suppose a matrix representation $\rho: G\rightarrow \mathrm{GL}(n)$
commutes with entrywise application of any function $\sigma:\mathbb{R}\rightarrow \mathbb{R}$
on $\mathbb{R}^n$, then every $\rho(g)$ is a permutation matrix.
\end{lemma}
In other words, the action of $G$ on $\mathbb{R}^n$ is induced from 
some action of $G$ on the \textit{set} $X:= \{1, \dotsb, n\}$.
Lemma~\ref{lem:nonlin-permute} follows from known results in the literature \cite{WOOD1996,pacini2024a,godfrey-2022-symmetries}, we provide a self-contained simple proof in the Supplementary Material.

Note that it has been shown \cite{ravanbakhshUniversalEquivariantMultilayer2020}
that for universal \textit{approximation} of $G$-equivariant maps using
MLPs with one hidden layer, it is enough to take $X = \bigsqcup_{s=1}^n G$. 
That is, the regular $G$-set alone is enough.
It is also known that not all $G$-sets yield universal approximation~\cite{pacini2025on}.
It would be of independent interest
to understand which combination of homogeneous spaces realizes 
approximations of $G$-equivariant functions most efficiently.

\subsection{Multi-head Self-Attention}
\label{sec:attn}
Our equivariant attention layers will be maps
$
	\mathrm{attn}: \mathbb{R}^{\mathcal{H}}\otimes V \rightarrow \mathbb{R}^{\mathcal{H}}\otimes V
$
that are equivariant to the $G$-action on $\mathbb{R}^{\mathcal{H}}\otimes V$.
Note that $G$ acts on both factors of the tensor product, 
but only equivariance on the second factor is nontrivial, since 
self-attention is permutation-equivariant on $\mathcal{H}$.

Contrary to
Refs.~\cite{bokmanFloppingFLOPsLeveraging2025,nordstromOcticVisionTransformers2025},
we will outline a general construction for equivariant self-attention and then
describe two special cases that are in some sense opposite to each other. 
The underlying principle that guarantees equivariance is to compute invariant
attention scores
\cite{assaadVNTransformerRotationEquivariantAttention2023,kunduSteerableTransformersVolumetric2025,nordstromOcticVisionTransformers2025}.
In fact, our construction can be summarized as \textit{ordinary multi-head
self-attention with respect to a $G$-invariant inner product and a $G$-stable
orthogonal decomposition of $V$}. This is elaborated in the following.

Equip $V$ with a $G$-invariant inner product $\braket{\cdot , \cdot}$, and
suppose $V = V_1 \oplus \dotsb \oplus V_h$ is an orthogonal decomposition of $V$
into subspaces stable under $G$. We will refer to these subspaces as attention
heads. 
Let $\phi_q, \phi_k, \phi_v\in \mathrm{Aff}_G(V,V)$ be learnable $G$-equivariant
affine maps. The raw attention scores $\alpha^{(r)}: \mathcal{H}\times \mathcal{H}\rightarrow \mathbb{R}$ in the $r$-th head are computed according to
\begin{equation}
\alpha^{(r)}(a,b) = \braket{\Pi_r\phi_q x(a), \Pi_r\phi_k x(b)},
\end{equation}
where $\Pi_r: V\rightarrow V_r$ is the orthogonal projection onto the $r$-th head.
The output token $y\in C(\mathcal{H}, V)$ in the $r$-th head is given by 
$
y^{(r)}(a) = \sum_{b\in\mathcal{H}}s^{(r)}(a, b) \Pi_r\phi_v x(b),
$
where $s^{(r)}(a,b)$ are the attention probabilities, obtained by taking softmax
of $\alpha^{(r)}$ over the second entry. The output token at $a\in \mathcal{H}$ is simply $\phi_o(y^{(1)}(a) \oplus \dotsb \oplus y^{(h)}(a))$, where $\phi_o \in \mathrm{Aff}_G(V,V)$ is a learnable output projection.

There are several possibilities for the choice of the orthogonal decomposition of $V$.
For example, we could decompose each irrep into $h_\rho$ heads, $V = \bigoplus_{\rho\in \widehat{G}} \bigoplus_{r=1}^{h_\rho}\mathbb{R}^{C_\rho/h_\rho}\otimes V_\rho$,
called \textit{irrep-wise} attention.
Another natural choice is to take $V_r = \bigoplus_{\rho\in \widehat{G}} \mathbb{R}^{C_\rho/h}\otimes V_\rho$, which we call \textit{coupled} attention.

How
expressive is our construction? More precisely, can all functions that are
$G$-equivariant and expressible using an ordinary self-attention layer be
expressed using a $G$-equivariant self-attention layer presented in this
section? We answer this question affirmatively in the single-head case
(see the Supplementary Material for the proof):
\begin{theorem}
\label{thm:attn-rep}
    Let $V$ be an orthogonal $G$-representation. Let $\mathrm{attn}:
    \mathbb{R}^L\otimes V\rightarrow \mathbb{R}^L\otimes V$ be a $G$-equivariant
    function representable by one layer of single-head ordinary self-attention,
    potentially with biases in the query, key, and value maps.
    Then $\mathrm{attn}$ is representable by one layer of $G$-equivariant
    single-head self-attention.
    That is, the query, key, and value maps can be taken to be $G$-equivariant.
\end{theorem}

\subsection{Class Token and Invariantization}
For classification, which is the task considered in our experiments in
Section~\ref{sec:experiments}, we append a class token $\mathrm{cls}\in V$ right
before the first transformer block. Mathematically, this means we replace
$\mathcal{H}$ with $\mathcal{H}\sqcup \{\star\}$, where $\star$ is a single point
that is invariant under the $G$-action, and the class token is the token at
$\star$. For this procedure to be $G$-equivariant, the initial class token itself has to
belong to the invariant subspace of $V$. That is, it can only be nonzero in the
trivial representation.
Note that after passing through an attention layer, the non-trivial parts of the class token will in general be nonzero via interaction with the other tokens.
The class token is plucked out after the final 
transformer block and its features are used in a final linear layer for
classification. Since one expects the classification model to be
\textit{invariant}, that is, the output logits should remain the same if the
input image is transformed by a group element, the class token must be
\textit{invariantized} before the linear classification head.

Following Ref.~\cite{nordstromOcticVisionTransformers2025},
we use the following map for invariantization:
\begin{equation}
	\label{eq:invariant}
	V \ni \mathrm{cls} \mapsto \mathrm{cls}^{\rm{trivial}}\oplus \bigoplus_{\rho\in \widehat G, \rho\neq \rm{trivial}}\lVert \mathrm{cls}^{\rho}\rVert_{V_\rho},
\end{equation}
which is followed by concatenation along the channel dimension. Here,
$\lVert\cdot \rVert_{V_\rho}$ is any $G$-invariant norm on $V_\rho$. Since we
choose all representation matrices to be orthogonal in practice, we can simply
use the $L^2$ norm.


\subsection{The Embedding Theorem}
In this section, we will show that, roughly speaking, ``in a fixed network
architecture, more equivariance means less expressivity''. In slightly more
technical terms, the map that takes a discrete group $G\le \mathrm{O}(2)$
and sends it to the set of functions expressible by a $G$-equivariant
transformer of fixed depth is inclusion-reversing.

Let $\mathcal{F}_G\left[\delta, h; (C_\rho)_{\rho\in \widehat{G}}, (n_{\alpha})_{\alpha \in \mathrm{Sub}(G)/\sim}\right]$ denote the set of functions 
\begin{equation}
\text{E-ViT}: \mathbb{R}^{\mathcal{H}_0}\otimes \mathbb{R}^3 \rightarrow \mathbb{R}^{\mathcal{H}}
\otimes \left[\bigoplus_{\rho\in \widehat{G}}\mathbb{R}^{C_\rho}\otimes V_\rho\right]
\end{equation}
expressible as a composition
$
\mathrm{Block}_\delta \circ \dotsb \circ \mathrm{Block}_1 \circ \mathrm{PosEnc} \circ \mathrm{PE},
$
where $\mathrm{PE}$ and $\mathrm{PosEnc}$ are the patch embedding and positional encoding
layers described in Sec.~\ref{sec:pe_pos_enc}, and each $\mathrm{Block}_k$ is a
transformer block with $h$-head coupled self-attention (see Sec.~\ref{sec:attn}) whose
MLP involves $n_\alpha$ copies of the $\alpha$-th homogeneous space (see
Sec.~\ref{sec:nonlin}).

\begin{theorem}
\label{thm:expressivity}
Let $H\le G$ be a subgroup. Fix an $H$-equivariant linear isometric bijection
\begin{equation}
    \mathrm{Res}^G_H: 
    \underbrace{\bigoplus_{\rho\in \widehat{G}}\mathbb{R}^{C_\rho}\otimes V_\rho}_{=: V}
    \xrightarrow{\sim}
    \underbrace{\bigoplus_{\sigma\in \widehat{H}}\mathbb{R}^{D_\sigma}\otimes W_\sigma.}_{=:W}
\end{equation}
Note that the multiplicities $D_\sigma$ are uniquely determined. Then 
\begin{equation}
\begin{aligned}
&(\mathrm{Id}\otimes \mathrm{Res}^G_H)\circ \mathcal{F}_G\left[\delta,h; (C_\rho)_{\rho\in \widehat{G}}, (n_{\alpha})_{\alpha \in \mathrm{Sub}(G)/\sim}\right]\\
&\subset
\mathcal{F}_H\left[\delta,h; (D_\sigma)_{\sigma\in \widehat{H}}, (m_{\beta})_{\beta \in \mathrm{Sub}(H)/\sim}\right],
\end{aligned}
\end{equation}
where the numbers of homogeneous space copies $m_\beta$ are determined as follows:
the $G$-space 
\begin{equation}
X:= \bigsqcup_{\alpha\in \mathrm{Sub}(G)/\sim}
\bigsqcup_{s=1}^{n_\alpha} X_\alpha
\end{equation}
decomposes into a disjoint union of $H$-orbits. $m_\beta$
is then the number of times the $\beta$-th homogeneous $H$-space appears in $X$.

Moreover, if $\dim \mathrm{Hom}_H(\mathbb{R}^U, W) > \dim \mathrm{Hom}_G(\mathbb{R}^U, V)$, then the inclusion is strict.
\end{theorem}
The proof of Theorem~\ref{thm:expressivity} is in the Supplementary Material.

Apart from providing a rigorous formulation of the bias-expressivity tradeoff 
in the context of equivariant ViTs, Theorem~\ref{thm:expressivity} 
is also of practical value: it allows one to view a $G$-equivariant ViT
as an $H$-equivariant ViT, opening the door to gradual symmetry-breaking, for
example, by enforcing $G$-equivariance in early training epochs, and then
procedurally relaxing to smaller subgroups at later stages.

\section{Experiments}
\label{sec:experiments}
In this section, we carry out experiments with equivariant ViTs for subgroups of
$D_4$ and $D_6$ on the PatternNet dataset. For the models equivariant to $D_6$
and its subgroups, we consider the hexagonal lattice structure for the
underlying $G$-patchified grid (see Fig.~\ref{fig:d6-filters}).

The PatternNet data set \cite{zhouPatternNetBenchmarkDataset2018} consists of
$\text{30,400}$ aerial images divided into 38 classes, each of which has $800$
images. We split the images within each class into $80\%$ training and $20\%$
validation data. All models are trained using AdamW with (unweighted) cross
entropy as loss function, and evaluated using mean accuracy as the primary metric.
We refer the reader to the Supplementary Material for further experiment details.

\subsection{Does equivariance matter?}
\label{sec:exp-equivariance}

Our first experiment aims to test the hypothesis 
that equivariance (more precisely invariance since the task is classification) implies better sample efficiency. This is intuitively plausible:
for a $G$-invariant model, a single labeled example $(x, y)$ is
equivalent to $|G|$ examples, namely $\{(gx, y) | g\in G\}$.
In other words, the model has built-in data augmentation.

To this end, we train $G$-invariant classifiers on $10\%, 40\%,$ and $100\%$ of
the training data for $G=D_4, C_4, D_2, C_1$ on the square grid and $G=D_6, C_6,
D_3, C_1$ on the hexagonal grid. The same set of data is always used at each
sampling fraction. We adjust the maximum training epoch and early stopping
patience to roughly compensate for the reduced number of training examples.

Coupled equivariant attention with $h=3$ attention heads is employed in this experiment. For
all models, we take the feature space $V$ to be $3q$ copies of the regular
representation with $q=1,2,\dotsb$, so that each head in the attention layer has
$q$ copies of the regular representation. The results are shown in
Fig.~\ref{fig:exp_equivariance}.
Note that there are significantly more runs with sample ratio $=0.1$
due to high variance.

\begin{figure}[htbp]
\centering
\includegraphics[width=.98\textwidth]{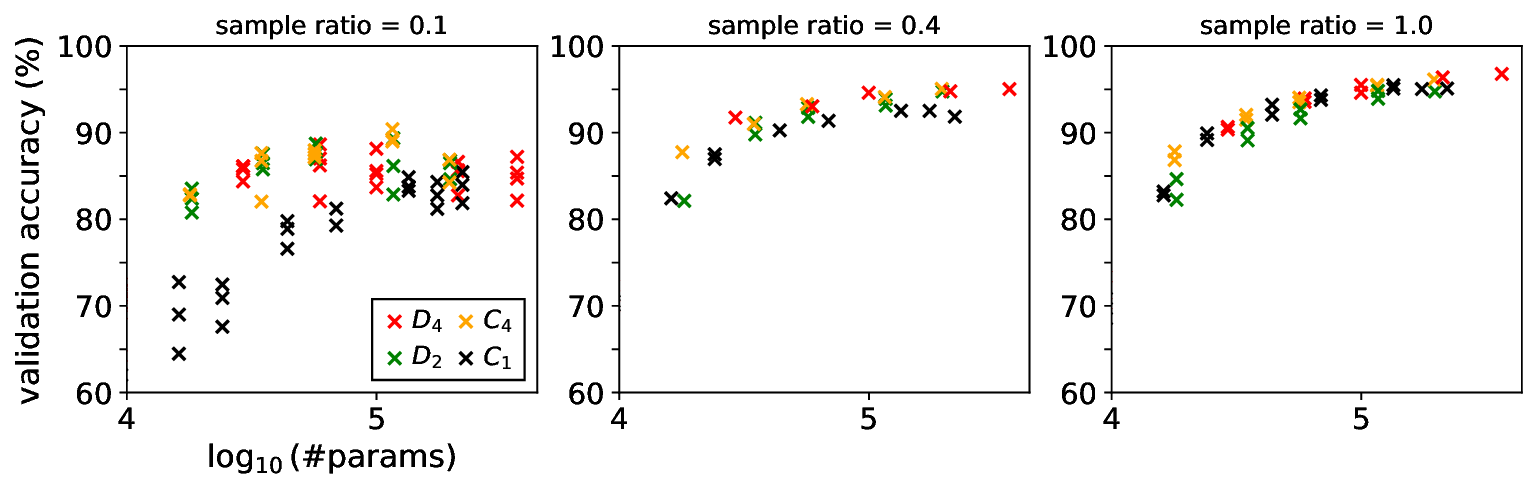}

\includegraphics[width=.98\textwidth]{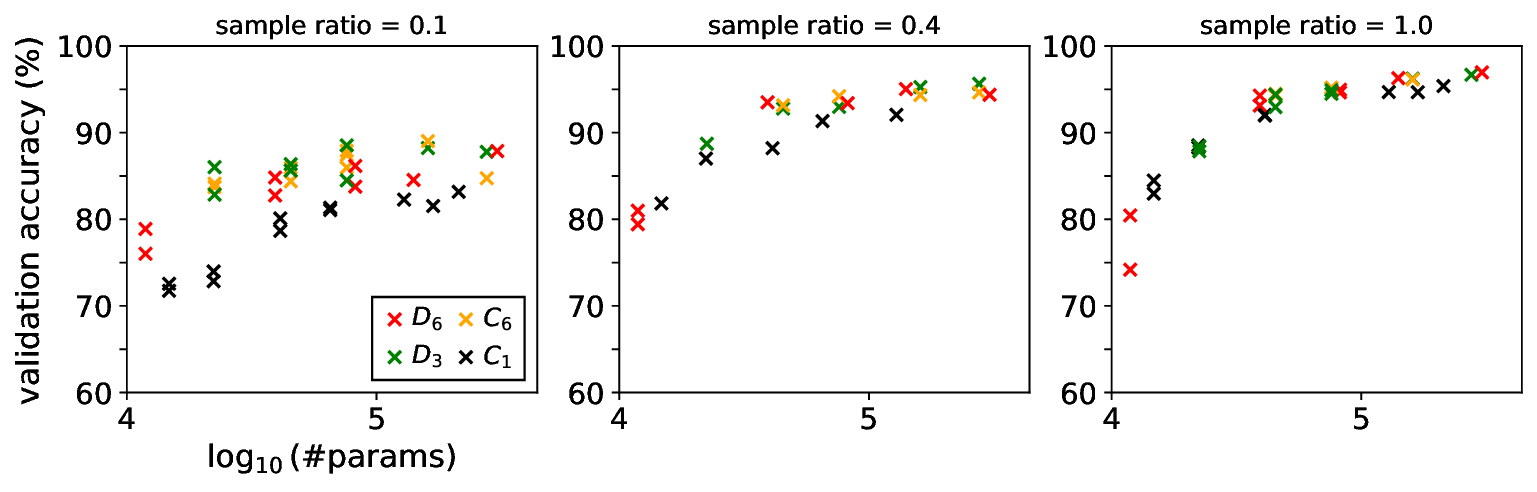}

\caption{Top: equivariant models on the square grid. Bottom: equivariant models on
the hexagonal grid.}
\label{fig:exp_equivariance}
\end{figure}

\subsection{Comparison of attention mechanisms}
\label{sec:exp-attn-type}
Here, we compare the two attention mechanisms detailed in 
Section~\ref{sec:attn}: irrep-wise and coupled. For this experiment, we take $G=
C_4$ with $4q$ copies of the regular representation, with $q=1, 2, 3,\dotsb$,
and we use four attention heads in total for both attention types (see
Table~\ref{tab:attn_types}). The attention heads are configured such that there
are always $4q$ features per token per head (note that $\rmE_1$ is
two-dimensional). Two sets of experiments are carried out:
one with only $10\%$ of the training data, and one with the complete training data.
The results are shown in Fig.~\ref{fig:exp_attn_type}.

\begin{table}[htbp]
\centering
\begin{tabular}{|c| c |c|}
\hline
    & coupled & irrepwise \\ 
    \hline
channel dimension & \multicolumn{2}{c|}{$C_{\rmA}= C_{\rmB} = C_{\rmE_1}= 4q$} \\
\hline
irreps in each head & $\rmA^{\oplus q} \oplus \rmB^{\oplus q} \oplus \rmE_1^{\oplus q}$  (all heads) &  
    \makecell{$\rmA^{\oplus 4q}$\\$ \rmB^{\oplus 4q}$\\$\rmE^{\oplus 2q}$\\$\rmE^{\oplus 2q}$}\\
\hline
\end{tabular}
\vspace{.5em}
\caption{Chosen configuration of attention heads for the two attention mechanisms for the $C_4$-equivariant transformer.}
\label{tab:attn_types}
\end{table}

\begin{figure}[htbp]
\centering
\includegraphics[width=.7\textwidth]{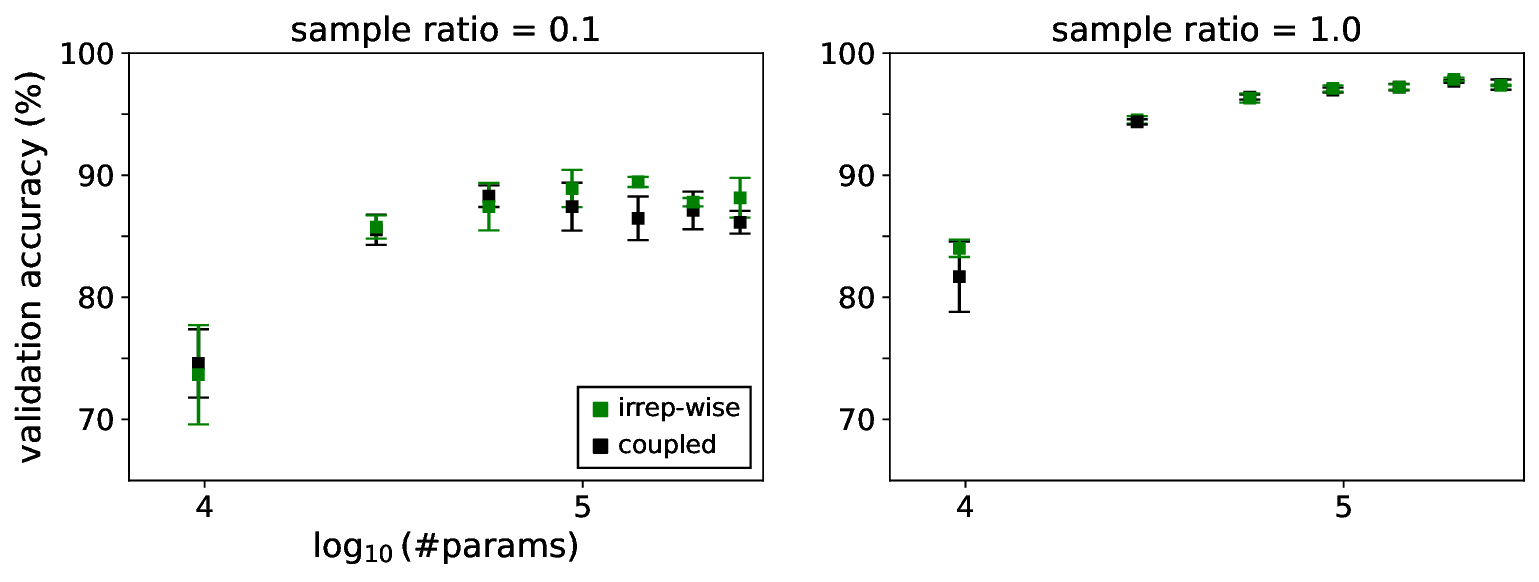}
\vspace{-1em}
\caption{Comparison of irrep-wise and coupled attention 
with 10\% (left) and 100\% (right) of the training data.
}
\label{fig:exp_attn_type}
\end{figure}

\subsection{Comparison of homogeneous space combinations in MLP}
\label{sec:exp-homogs}
Finally, we study the performance differences resulting from different
choices of the $G$-set $X$ in our construction of the nonlinear layer (see Section~\ref{sec:nonlin}). 
We take $G=D_4$ with $3q$ copies of the regular representation as the feature
space $V$, and consider here three families of $G$-sets, each formed by $n$
copies of $D_4$ (regular action), $n$ copies of $D_4$ plus $8$ points (trivial
action), and $n$ copies of $D_4/\braket{r} \sqcup D_4/\braket{t} \sqcup
D_4/\braket{tr, r^2}$ (see the Supplementary Material for a list of homogeneous
spaces of $D_4$). These are chosen so that all irreps appear at least once in
$X$, preventing information loss in the MLP layer.
Note that the second and third choices contain higher ratios of $\rmA_1$
features in the hidden layer ($\frac{9}{16}$ and $\frac{3}{8}$ respectively, as opposed to $\frac{1}{8}$).

\begin{figure}[htbp]
\centering
\includegraphics[width=.95\textwidth]{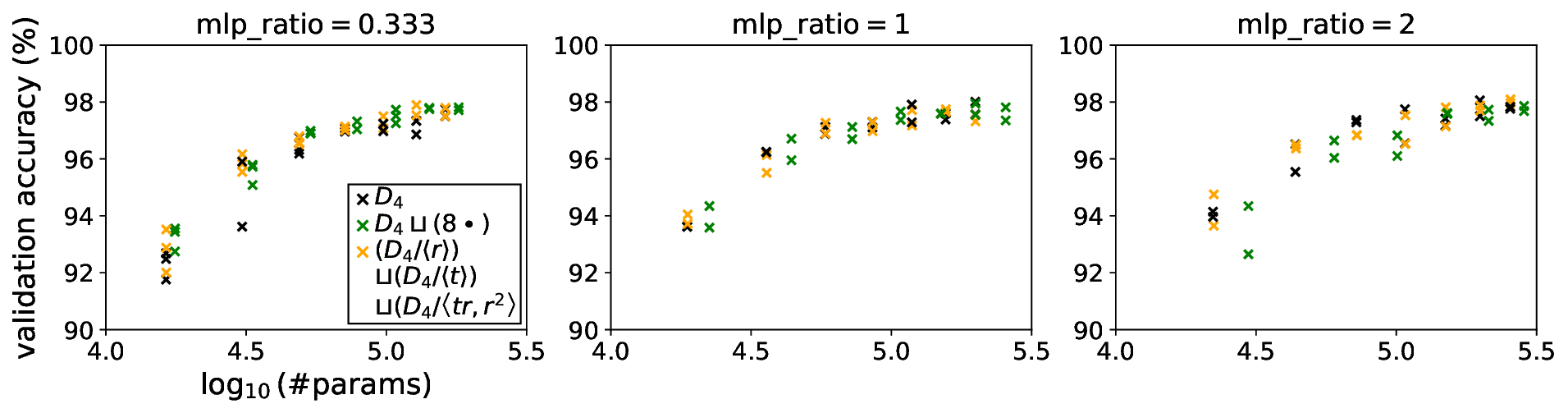}
\vspace{-1em}
\caption{Performance comparison of $D_4$-equivariant ViTs with different combinations of homogeneous spaces in the MLP layers.}
\label{fig:exp_mlp}
\end{figure}

Instead of varying the number of homogeneous spaces $n$ independently of $q$, we
fix three ``MLP ratios'', defined as $\frac{n}{3q}$.
The findings are summarized in Fig.~\ref{fig:exp_mlp}.

\subsection{Discussion}
Our first experiment (Section~\ref{sec:exp-equivariance}) confirms the intuition
that the significance of equivariance is magnified in the low-data
regime: while the nonequivariant ViT performs significantly worse than the
equivariant counterparts when trained on $10\%$ of the data, all ViT variants
are practically indistinguishable when all of the training data is used.
In fact, the $D_2$-equivariant model seems to perform slightly worse than the
nonequivariant one. 
However, it is not clear whether more equivariance always entails better
performance: at $10\%$ sample ratio, the $D_3$-equivariant model 
achieves slightly better accuracies than the $D_6$-equivariant counterpart
at fixed parameter counts.

The last two experiments show that the choice of attention type
(Section~\ref{sec:exp-attn-type}) and of $G$-set $X$ for the MLP layers
(Section~\ref{sec:exp-homogs}) do not affect performance in a drastic way. We can
only conclude that marginal accuracy gain is achieved with higher fractions of
$\rmA_1$ features in the MLP hidden layer (Fig.~\ref{fig:exp_mlp}) when the MLP
ratio is low, and with irrep-wise attention as opposed to coupled attention (see
Fig.~\ref{fig:exp_attn_type}).
Nevertheless, one cannot rule out the possibility that perturbations in
other hyperparameters might change this conclusion.

One potential explanation for the observed worse performance with larger
symmetry groups is that for classification tasks, especially for
``low-frequency'' images, in the sense of Fourier transforms, like those in
PatternNet (e.g. beaches, baseball fields, runways), non-trivial irrep features
are not as crucial. This leads us to the following conjecture:
\begin{conjecture}
A $G$-invariant classification model with a $G$-equivariant ViT backbone mostly
uses features from the trivial representation.
\end{conjecture}
If this is true, then it might be more difficult for a $D_4$-equivariant model
to learn invariant features compared to a $D_2$-equivariant one if the
respective regular representations are used as feature spaces in both models, as the 
fraction of invariant features in $V$ is $1/|G|$.
This could also explain the apparent advantage of irrep-wise attention over
coupled attention (Section~\ref{sec:exp-attn-type}): if little useful information
is stored in nontrivial irreps, mixing these irreps with the trivial one in the
same attention head could result in noisy attention scores.


\section{Limitations and Future Work}

Due to large hyperparameter search space and limited computational resources,
all experimental results presented here involve very small models ($\lesssim$
0.5M parameters), trained on a relatively small dataset (PatternNet). It is not
clear whether similar results hold at scale. Indeed, as our experiments suggest,
equivariance could be most important when data are scarce.
However, in Ref.~\cite{bokmanFloppingFLOPsLeveraging2025}, the argument is made that equivariance can also be practical in large dataset regimes, due to the increased sparsification of the linear layers when increasing the size of the group $G$ while keeping the total channel dimension fixed.
At the model sizes used in this paper, the computational benefit is not visible, but we hope that our implementations will be useful to further characterize the scaling properties of equivariant ViTs in future work.

In addition to scaling up, there are multiple orthogonal avenues for future
exploration. First, it is in principle possible to construct ViTs equivariant to
groups $G\le \mathrm{O}(2)$ that are not subgroups of $D_4$ or $D_6$. However,
their construction will inevitably involve more irregular grids (c.f.
Definition~\ref{defn:patches}), and it would be interesting to understand the
trade-off between exact discrete equivariance and approximate continuous
equivariance.

Second, as already remarked after Theorem~\ref{thm:expressivity}, one can break
the symmetry, either ``statically'', by concatenating transformer blocks in a
way that later blocks have smaller symmetry groups , or ``dynamically'', 
by imposing a large symmetry group at early training stages, 
and then gradually relaxing the symmetry group to smaller ones.
We expect this kind of architectures or training schedule to be advantageous for 
datasets that do not respect rotational symmetries exactly.

Finally, for a fixed group $G\le \mathrm{O}(2)$, each choice of feature space $V$ must be independently tested in the current formulation.
It would be of great benefit for the practical use of the networks presented in this paper, and for group equivariant neural networks in general, to find a principled way to pick $V$ for a given task, or to optimize the choice as part of training a network.

\;\\
\noindent
\textbf{Code Availability.} The code used in this study will be made publicly available in a permanent online repository upon publication of the article.

\bibliographystyle{plainnat}

\bibliography{Refs, refs_editable}

@misc{eijkelboomEquivariantMessagePassing2023,
	title = {E(n) {Equivariant} {Message} {Passing} {Simplicial} {Networks}},
	url = {http://arxiv.org/abs/2305.07100},
	doi = {10.48550/arXiv.2305.07100},
	urldate = {2026-03-04},
	publisher = {arXiv},
	author = {Eijkelboom, Floor and Hesselink, Rob and Bekkers, Erik},
	month = oct,
	year = {2023},
	note = {arXiv:2305.07100 [cs]},
}

@misc{weiler3DSteerableCNNs2018,
	title = {{3D} {Steerable} {CNNs}: {Learning} {Rotationally} {Equivariant} {Features} in {Volumetric} {Data}},
	shorttitle = {{3D} {Steerable} {CNNs}},
	url = {http://arxiv.org/abs/1807.02547},
	doi = {10.48550/arXiv.1807.02547},
	urldate = {2026-03-01},
	publisher = {arXiv},
	author = {Weiler, Maurice and Geiger, Mario and Welling, Max and Boomsma, Wouter and Cohen, Taco},
	month = oct,
	year = {2018},
	note = {arXiv:1807.02547 [cs]},
}

@inproceedings{heEfficientEquivariantNetwork2021,
	title = {Efficient {Equivariant} {Network}},
	volume = {34},
	url = {https://proceedings.neurips.cc/paper/2021/hash/2a79ea27c279e471f4d180b08d62b00a-Abstract.html},
	urldate = {2026-02-24},
	booktitle = {Advances in {Neural} {Information} {Processing} {Systems}},
	publisher = {Curran Associates, Inc.},
	author = {He, Lingshen and Chen, Yuxuan and shen, zhengyang and Dong, Yiming and Wang, Yisen and Lin, Zhouchen},
	year = {2021},
	pages = {5290--5302},
}

@misc{fuchsSE3Transformers3DRotoTranslation2020,
	title = {{SE}(3)-{Transformers}: {3D} {Roto}-{Translation} {Equivariant} {Attention} {Networks}},
	shorttitle = {{SE}(3)-{Transformers}},
	url = {http://arxiv.org/abs/2006.10503},
	doi = {10.48550/arXiv.2006.10503},
	urldate = {2026-01-20},
	publisher = {arXiv},
	author = {Fuchs, Fabian B. and Worrall, Daniel E. and Fischer, Volker and Welling, Max},
	month = nov,
	year = {2020},
	note = {arXiv:2006.10503 [cs]},
}

@inproceedings{shawe-taylorBuildingSymmetriesFeedforward1989,
	title = {Building symmetries into feedforward networks},
	url = {https://ieeexplore.ieee.org/document/51951},
	urldate = {2026-03-01},
	booktitle = {1989 {First} {IEE} {International} {Conference} on {Artificial} {Neural} {Networks}, ({Conf}. {Publ}. {No}. 313)},
	author = {Shawe-Taylor, J.},
	month = oct,
	year = {1989},
	pages = {158--162},
}

@misc{assaadVNTransformerRotationEquivariantAttention2023,
	title = {{VN}-{Transformer}: {Rotation}-{Equivariant} {Attention} for {Vector} {Neurons}},
	shorttitle = {{VN}-{Transformer}},
	url = {http://arxiv.org/abs/2206.04176},
	doi = {10.48550/arXiv.2206.04176},
	urldate = {2026-03-01},
	publisher = {arXiv},
	author = {Assaad, Serge and Downey, Carlton and Al-Rfou, Rami and Nayakanti, Nigamaa and Sapp, Ben},
	month = jan,
	year = {2023},
	note = {arXiv:2206.04176 [cs]},
}

@misc{bekkersRotoTranslationCovariantConvolutional2018,
	title = {Roto-{Translation} {Covariant} {Convolutional} {Networks} for {Medical} {Image} {Analysis}},
	url = {https://arxiv.org/abs/1804.03393v3},
	language = {en},
	urldate = {2026-02-20},
	journal = {arXiv.org},
	author = {Bekkers, Erik J. and Lafarge, Maxime W. and Veta, Mitko and Eppenhof, Koen AJ and Pluim, Josien PW and Duits, Remco},
	month = apr,
	year = {2018},
}

@article{zhouPatternNetBenchmarkDataset2018,
	title = {{PatternNet}: {A} {Benchmark} {Dataset} for {Performance} {Evaluation} of {Remote} {Sensing} {Image} {Retrieval}},
	volume = {145},
	issn = {09242716},
	shorttitle = {{PatternNet}},
	url = {http://arxiv.org/abs/1706.03424},
	doi = {10.1016/j.isprsjprs.2018.01.004},
	urldate = {2026-02-19},
	journal = {ISPRS Journal of Photogrammetry and Remote Sensing},
	author = {Zhou, Weixun and Newsam, Shawn and Li, Congmin and Shao, Zhenfeng},
	month = nov,
	year = {2018},
	note = {arXiv:1706.03424 [cs]},
	pages = {197--209},
}

@misc{batatiaMACEHigherOrder2023,
	title = {{MACE}: {Higher} {Order} {Equivariant} {Message} {Passing} {Neural} {Networks} for {Fast} and {Accurate} {Force} {Fields}},
	shorttitle = {{MACE}},
	url = {http://arxiv.org/abs/2206.07697},
	doi = {10.48550/arXiv.2206.07697},
	urldate = {2026-01-30},
	publisher = {arXiv},
	author = {Batatia, Ilyes and Kovács, Dávid Péter and Simm, Gregor N. C. and Ortner, Christoph and Csányi, Gábor},
	month = jan,
	year = {2023},
	note = {arXiv:2206.07697 [stat]},
}

@misc{kondorClebschGordanNetsFully2018,
	title = {Clebsch-{Gordan} {Nets}: a {Fully} {Fourier} {Space} {Spherical} {Convolutional} {Neural} {Network}},
	shorttitle = {Clebsch-{Gordan} {Nets}},
	url = {http://arxiv.org/abs/1806.09231},
	doi = {10.48550/arXiv.1806.09231},
	urldate = {2026-01-13},
	publisher = {arXiv},
	author = {Kondor, Risi and Lin, Zhen and Trivedi, Shubhendu},
	month = nov,
	year = {2018},
	note = {arXiv:1806.09231 [stat]},
}

@misc{andersonCormorantCovariantMolecular2019,
	title = {Cormorant: {Covariant} {Molecular} {Neural} {Networks}},
	shorttitle = {Cormorant},
	url = {http://arxiv.org/abs/1906.04015},
	doi = {10.48550/arXiv.1906.04015},
	urldate = {2026-01-24},
	publisher = {arXiv},
	author = {Anderson, Brandon and Hy, Truong-Son and Kondor, Risi},
	month = nov,
	year = {2019},
	note = {arXiv:1906.04015 [physics]},
}

@misc{maronInvariantEquivariantGraph2019,
	title = {Invariant and {Equivariant} {Graph} {Networks}},
	url = {http://arxiv.org/abs/1812.09902},
	doi = {10.48550/arXiv.1812.09902},
	language = {en},
	urldate = {2026-01-20},
	publisher = {arXiv},
	author = {Maron, Haggai and Ben-Hamu, Heli and Shamir, Nadav and Lipman, Yaron},
	month = apr,
	year = {2019},
	note = {arXiv:1812.09902 [cs]},
}

@article{batznerE3equivariantGraphNeural2022,
	title = {E(3)-equivariant graph neural networks for data-efficient and accurate interatomic potentials},
	volume = {13},
	copyright = {2022 The Author(s)},
	issn = {2041-1723},
	url = {https://www.nature.com/articles/s41467-022-29939-5},
	doi = {10.1038/s41467-022-29939-5},
	language = {en},
	number = {1},
	urldate = {2026-01-30},
	journal = {Nature Communications},
	publisher = {Nature Publishing Group},
	author = {Batzner, Simon and Musaelian, Albert and Sun, Lixin and Geiger, Mario and Mailoa, Jonathan P. and Kornbluth, Mordechai and Molinari, Nicola and Smidt, Tess E. and Kozinsky, Boris},
	month = may,
	year = {2022},
	pages = {2453},
}

@misc{satorrasEquivariantGraphNeural2022,
	title = {E(n) {Equivariant} {Graph} {Neural} {Networks}},
	url = {http://arxiv.org/abs/2102.09844},
	doi = {10.48550/arXiv.2102.09844},
	urldate = {2026-01-24},
	publisher = {arXiv},
	author = {Satorras, Victor Garcia and Hoogeboom, Emiel and Welling, Max},
	month = feb,
	year = {2022},
	note = {arXiv:2102.09844 [cs]},
}

@misc{chatzipantazisSE3EquivariantAttentionNetworks2023,
	title = {{SE}(3)-{Equivariant} {Attention} {Networks} for {Shape} {Reconstruction} in {Function} {Space}},
	url = {http://arxiv.org/abs/2204.02394},
	doi = {10.48550/arXiv.2204.02394},
	urldate = {2026-01-28},
	publisher = {arXiv},
	author = {Chatzipantazis, Evangelos and Pertigkiozoglou, Stefanos and Dobriban, Edgar and Daniilidis, Kostas},
	month = feb,
	year = {2023},
	note = {arXiv:2204.02394 [cs]},
}

@misc{taiEquivariantTransformerNetworks2019,
	title = {Equivariant {Transformer} {Networks}},
	url = {http://arxiv.org/abs/1901.11399},
	doi = {10.48550/arXiv.1901.11399},
	urldate = {2026-01-23},
	publisher = {arXiv},
	author = {Tai, Kai Sheng and Bailis, Peter and Valiant, Gregory},
	month = may,
	year = {2019},
	note = {arXiv:1901.11399 [cs]},
}

@misc{kondorGeneralizationEquivarianceConvolution2018,
	title = {On the {Generalization} of {Equivariance} and {Convolution} in {Neural} {Networks} to the {Action} of {Compact} {Groups}},
	url = {http://arxiv.org/abs/1802.03690},
	doi = {10.48550/arXiv.1802.03690},
	urldate = {2026-01-21},
	publisher = {arXiv},
	author = {Kondor, Risi and Trivedi, Shubhendu},
	month = nov,
	year = {2018},
	note = {arXiv:1802.03690 [stat]},
}

@misc{veelingRotationEquivariantCNNs2018,
	title = {Rotation {Equivariant} {CNNs} for {Digital} {Pathology}},
	url = {http://arxiv.org/abs/1806.03962},
	doi = {10.48550/arXiv.1806.03962},
	urldate = {2026-02-07},
	publisher = {arXiv},
	author = {Veeling, Bastiaan S. and Linmans, Jasper and Winkens, Jim and Cohen, Taco and Welling, Max},
	month = jun,
	year = {2018},
	note = {arXiv:1806.03962 [cs]},
}

@misc{cohenGeneralTheoryEquivariant2020,
	title = {A {General} {Theory} of {Equivariant} {CNNs} on {Homogeneous} {Spaces}},
	url = {http://arxiv.org/abs/1811.02017},
	doi = {10.48550/arXiv.1811.02017},
	urldate = {2026-01-13},
	publisher = {arXiv},
	author = {Cohen, Taco and Geiger, Mario and Weiler, Maurice},
	month = jan,
	year = {2020},
	note = {arXiv:1811.02017 [cs]},
}

@misc{cohenGroupEquivariantConvolutional2016,
	title = {Group {Equivariant} {Convolutional} {Networks}},
	url = {https://arxiv.org/abs/1602.07576v3},
	doi = {10.48550/arXiv.1602.07576},
	language = {en},
	urldate = {2023-02-13},
	journal = {arXiv.org},
	author = {Cohen, Taco S. and Welling, Max},
	month = feb,
	year = {2016},
}

@misc{weilerGeneral$E2$EquivariantSteerable2021,
	title = {General \${E}(2)\$-{Equivariant} {Steerable} {CNNs}},
	url = {http://arxiv.org/abs/1911.08251},
	doi = {10.48550/arXiv.1911.08251},
	urldate = {2026-01-20},
	publisher = {arXiv},
	author = {Weiler, Maurice and Cesa, Gabriele},
	month = apr,
	year = {2021},
	note = {arXiv:1911.08251 [cs]},
}

@misc{kabaEquivarianceLearnedCanonicalization2023,
	title = {Equivariance with {Learned} {Canonicalization} {Functions}},
	url = {http://arxiv.org/abs/2211.06489},
	doi = {10.48550/arXiv.2211.06489},
	urldate = {2026-01-28},
	publisher = {arXiv},
	author = {Kaba, Sékou-Oumar and Mondal, Arnab Kumar and Zhang, Yan and Bengio, Yoshua and Ravanbakhsh, Siamak},
	month = jul,
	year = {2023},
	note = {arXiv:2211.06489 [cs]},
}

@misc{xu$E2$EquivariantVisionTransformer2023,
	title = {\${E}(2)\$-{Equivariant} {Vision} {Transformer}},
	url = {http://arxiv.org/abs/2306.06722},
	doi = {10.48550/arXiv.2306.06722},
	urldate = {2026-01-14},
	publisher = {arXiv},
	author = {Xu, Renjun and Yang, Kaifan and Liu, Ke and He, Fengxiang},
	month = jul,
	year = {2023},
	note = {arXiv:2306.06722 [cs]},
}

@misc{bronsteinGeometricDeepLearning2021,
	title = {Geometric {Deep} {Learning}: {Grids}, {Groups}, {Graphs}, {Geodesics}, and {Gauges}},
	shorttitle = {Geometric {Deep} {Learning}},
	url = {http://arxiv.org/abs/2104.13478},
	doi = {10.48550/arXiv.2104.13478},
	urldate = {2025-12-05},
	publisher = {arXiv},
	author = {Bronstein, Michael M. and Bruna, Joan and Cohen, Taco and Veličković, Petar},
	month = may,
	year = {2021},
	note = {arXiv:2104.13478 [cs]},
}

@misc{bokmanFloppingFLOPsLeveraging2025,
	title = {Flopping for {FLOPs}: {Leveraging} equivariance for computational efficiency},
	shorttitle = {Flopping for {FLOPs}},
	url = {http://arxiv.org/abs/2502.05169},
	doi = {10.48550/arXiv.2502.05169},
	urldate = {2026-01-28},
	publisher = {arXiv},
	author = {Bökman, Georg and Nordström, David and Kahl, Fredrik},
	month = jun,
	year = {2025},
	note = {arXiv:2502.05169 [cs]},
}

@misc{kunduSteerableTransformersVolumetric2025,
	title = {Steerable {Transformers} for {Volumetric} {Data}},
	url = {http://arxiv.org/abs/2405.15932},
	doi = {10.48550/arXiv.2405.15932},
	urldate = {2026-01-28},
	publisher = {arXiv},
	author = {Kundu, Soumyabrata and Kondor, Risi},
	month = oct,
	year = {2025},
	note = {arXiv:2405.15932 [cs]},
}

@misc{islamPlatonicTransformersSolid2025,
	title = {Platonic {Transformers}: {A} {Solid} {Choice} {For} {Equivariance}},
	shorttitle = {Platonic {Transformers}},
	url = {http://arxiv.org/abs/2510.03511},
	doi = {10.48550/arXiv.2510.03511},
	urldate = {2026-01-14},
	publisher = {arXiv},
	author = {Islam, Mohammad Mohaiminul and Anand, Rishabh and Wessels, David R. and Kruiff, Friso de and Kuipers, Thijs P. and Ying, Rex and Sánchez, Clara I. and Vadgama, Sharvaree and Bökman, Georg and Bekkers, Erik J.},
	month = oct,
	year = {2025},
	note = {arXiv:2510.03511 [cs]},
}

@misc{jaderbergSpatialTransformerNetworks2016,
	title = {Spatial {Transformer} {Networks}},
	url = {http://arxiv.org/abs/1506.02025},
	doi = {10.48550/arXiv.1506.02025},
	urldate = {2026-01-24},
	publisher = {arXiv},
	author = {Jaderberg, Max and Simonyan, Karen and Zisserman, Andrew and Kavukcuoglu, Koray},
	month = feb,
	year = {2016},
	note = {arXiv:1506.02025 [cs]},
}

@misc{hoogeboomHexaConv2018,
	title = {{HexaConv}},
	url = {http://arxiv.org/abs/1803.02108},
	doi = {10.48550/arXiv.1803.02108},
	urldate = {2026-01-13},
	publisher = {arXiv},
	author = {Hoogeboom, Emiel and Peters, Jorn W. T. and Cohen, Taco S. and Welling, Max},
	month = mar,
	year = {2018},
	note = {arXiv:1803.02108 [cs]},
}

@misc{cohenSteerableCNNs2016,
	title = {Steerable {CNNs}},
	url = {http://arxiv.org/abs/1612.08498},
	doi = {10.48550/arXiv.1612.08498},
	urldate = {2026-01-14},
	publisher = {arXiv},
	author = {Cohen, Taco S. and Welling, Max},
	month = dec,
	year = {2016},
	note = {arXiv:1612.08498 [cs]},
}

@misc{nordstromOcticVisionTransformers2025,
	title = {Octic {Vision} {Transformers}: {Quicker} {ViTs} {Through} {Equivariance}},
	shorttitle = {Octic {Vision} {Transformers}},
	url = {http://arxiv.org/abs/2505.15441},
	doi = {10.48550/arXiv.2505.15441},
	urldate = {2026-01-28},
	publisher = {arXiv},
	author = {Nordström, David and Edstedt, Johan and Kahl, Fredrik and Bökman, Georg},
	month = sep,
	year = {2025},
	note = {arXiv:2505.15441 [cs]},
}

@misc{kunduGeometricApproachSteerable2025,
	title = {A {Geometric} {Approach} to {Steerable} {Convolutions}},
	url = {http://arxiv.org/abs/2510.18813},
	doi = {10.48550/arXiv.2510.18813},
	urldate = {2026-01-26},
	publisher = {arXiv},
	author = {Kundu, Soumyabrata and Kondor, Risi},
	month = oct,
	year = {2025},
	note = {arXiv:2510.18813 [cs]},
}

@misc{dosovitskiyImageWorth16x162021,
	title = {An {Image} is {Worth} 16x16 {Words}: {Transformers} for {Image} {Recognition} at {Scale}},
	shorttitle = {An {Image} is {Worth} 16x16 {Words}},
	url = {http://arxiv.org/abs/2010.11929},
	doi = {10.48550/arXiv.2010.11929},
	urldate = {2022-11-10},
	publisher = {arXiv},
	author = {Dosovitskiy, Alexey and Beyer, Lucas and Kolesnikov, Alexander and Weissenborn, Dirk and Zhai, Xiaohua and Unterthiner, Thomas and Dehghani, Mostafa and Minderer, Matthias and Heigold, Georg and Gelly, Sylvain and Uszkoreit, Jakob and Houlsby, Neil},
	month = jun,
	year = {2021},
	note = {arXiv:2010.11929 [cs]},
}

@misc{oquabDINOv2LearningRobust2024,
	title = {{DINOv2}: {Learning} {Robust} {Visual} {Features} without {Supervision}},
	shorttitle = {{DINOv2}},
	url = {http://arxiv.org/abs/2304.07193},
	doi = {10.48550/arXiv.2304.07193},
	urldate = {2026-02-08},
	publisher = {arXiv},
	author = {Oquab, Maxime and Darcet, Timothée and Moutakanni, Théo and Vo, Huy and Szafraniec, Marc and Khalidov, Vasil and Fernandez, Pierre and Haziza, Daniel and Massa, Francisco and El-Nouby, Alaaeldin and Assran, Mahmoud and Ballas, Nicolas and Galuba, Wojciech and Howes, Russell and Huang, Po-Yao and Li, Shang-Wen and Misra, Ishan and Rabbat, Michael and Sharma, Vasu and Synnaeve, Gabriel and Xu, Hu and Jegou, Hervé and Mairal, Julien and Labatut, Patrick and Joulin, Armand and Bojanowski, Piotr},
	month = feb,
	year = {2024},
	note = {arXiv:2304.07193 [cs]},
}

@article{zhangArtificialIntelligenceScience2025,
	title = {Artificial {Intelligence} for {Science} in {Quantum}, {Atomistic}, and {Continuum} {Systems}},
	volume = {18},
	issn = {1935-8237, 1935-8245},
	url = {http://arxiv.org/abs/2307.08423},
	doi = {10.1561/2200000115},
	number = {4},
	urldate = {2026-02-05},
	journal = {Foundations and Trends® in Machine Learning},
	author = {Zhang, Xuan and Wang, Limei and Helwig, Jacob and Luo, Youzhi and Fu, Cong and Xie, Yaochen and Liu, Meng and Lin, Yuchao and Xu, Zhao and Yan, Keqiang and Adams, Keir and Weiler, Maurice and Li, Xiner and Fu, Tianfan and Wang, Yucheng and Strasser, Alex and Yu, Haiyang and Xie, YuQing and Fu, Xiang and Xu, Shenglong and Liu, Yi and Du, Yuanqi and Saxton, Alexandra and Ling, Hongyi and Lawrence, Hannah and Stärk, Hannes and Gui, Shurui and Edwards, Carl and Gao, Nicholas and Ladera, Adriana and Wu, Tailin and Hofgard, Elyssa F. and Tehrani, Aria Mansouri and Wang, Rui and Daigavane, Ameya and Bohde, Montgomery and Kurtin, Jerry and Huang, Qian and Phung, Tuong and Xu, Minkai and Joshi, Chaitanya K. and Mathis, Simon V. and Azizzadenesheli, Kamyar and Fang, Ada and Aspuru-Guzik, Alán and Bekkers, Erik and Bronstein, Michael and Zitnik, Marinka and Anandkumar, Anima and Ermon, Stefano and Liò, Pietro and Yu, Rose and Günnemann, Stephan and Leskovec, Jure and Ji, Heng and Sun, Jimeng and Barzilay, Regina and Jaakkola, Tommi and Coley, Connor W. and Qian, Xiaoning and Qian, Xiaofeng and Smidt, Tess and Ji, Shuiwang},
	year = {2025},
	note = {arXiv:2307.08423 [cs]},
	pages = {385--912},
}

@article{gerkenGeometricDeepLearning2023,
	title = {Geometric deep learning and equivariant neural networks},
	volume = {56},
	issn = {1573-7462},
	url = {https://doi.org/10.1007/s10462-023-10502-7},
	doi = {10.1007/s10462-023-10502-7},
	language = {en},
	number = {12},
	urldate = {2026-01-20},
	journal = {Artificial Intelligence Review},
	author = {Gerken, Jan E. and Aronsson, Jimmy and Carlsson, Oscar and Linander, Hampus and Ohlsson, Fredrik and Petersson, Christoffer and Persson, Daniel},
	month = dec,
	year = {2023},
	pages = {14605--14662},
}

@inproceedings{ravanbakhshUniversalEquivariantMultilayer2020,
	title = {Universal {Equivariant} {Multilayer} {Perceptrons}},
	issn = {2640-3498},
	url = {https://proceedings.mlr.press/v119/ravanbakhsh20a.html},
	language = {en},
	urldate = {2026-01-14},
	booktitle = {Proceedings of the 37th {International} {Conference} on {Machine} {Learning}},
	publisher = {PMLR},
	author = {Ravanbakhsh, Siamak},
	month = nov,
	year = {2020},
	pages = {7996--8006},
}

@misc{vaswaniAttentionAllYou2017,
	title = {Attention {Is} {All} {You} {Need}},
	url = {http://arxiv.org/abs/1706.03762},
	doi = {10.48550/arXiv.1706.03762},
	urldate = {2022-11-10},
	publisher = {arXiv},
	author = {Vaswani, Ashish and Shazeer, Noam and Parmar, Niki and Uszkoreit, Jakob and Jones, Llion and Gomez, Aidan N. and Kaiser, Lukasz and Polosukhin, Illia},
	month = dec,
	year = {2017},
	note = {arXiv:1706.03762 [cs]},
}

@InProceedings{Wang_2025_CVPR,
    author    = {Wang, Jianyuan and Chen, Minghao and Karaev, Nikita and Vedaldi, Andrea and Rupprecht, Christian and Novotny, David},
    title     = {VGGT: Visual Geometry Grounded Transformer},
    booktitle = {Proceedings of the IEEE/CVF Conference on Computer Vision and Pattern Recognition (CVPR)},
    month     = {June},
    year      = {2025},
    pages     = {5294-5306}
}

@BOOK{serreLinearRepresentations,
  title     = "Linear representations of finite groups",
  author    = "Serre, Jean-Pierre",
  publisher = "Springer",
  series    = "Graduate Texts in Mathematics",
  month     =  sep,
  year      =  1977,
  address   = "New York, NY",
  copyright = "https://www.springernature.com/gp/researchers/text-and-data-mining",
  language  = "en"
}

@inproceedings{
pacini2024a,
title={A Characterization Theorem for Equivariant Networks with Point-wise Activations},
author={Marco Pacini and Xiaowen Dong and Bruno Lepri and Gabriele Santin},
booktitle={The Twelfth International Conference on Learning Representations},
year={2024},
url={https://openreview.net/forum?id=79FVDdfoSR}
}

@article{WOOD1996,
title = {Representation theory and invariant neural networks},
journal = {Discrete Applied Mathematics},
volume = {69},
number = {1},
pages = {33-60},
year = {1996},
issn = {0166-218X},
doi = {https://doi.org/10.1016/0166-218X(95)00075-3},
url = {https://www.sciencedirect.com/science/article/pii/0166218X95000753},
author = {Jeffrey Wood and John Shawe-Taylor}
}

@inproceedings{godfrey-2022-symmetries,
 author = {Godfrey, Charles and Brown, Davis and Emerson, Tegan and Kvinge, Henry},
 booktitle = {Advances in Neural Information Processing Systems},
 editor = {S. Koyejo and S. Mohamed and A. Agarwal and D. Belgrave and K. Cho and A. Oh},
 pages = {11893--11905},
 publisher = {Curran Associates, Inc.},
 title = {On the Symmetries of Deep Learning Models and their Internal Representations},
 url = {https://proceedings.neurips.cc/paper_files/paper/2022/file/4df3510ad02a86d69dc32388d91606f8-Paper-Conference.pdf},
 volume = {35},
 year = {2022}
}

@inproceedings{
pacini2025on,
title={On Universality Classes of Equivariant Networks},
author={Marco Pacini and Gabriele Santin and Bruno Lepri and Shubhendu Trivedi},
booktitle={The Thirty-ninth Annual Conference on Neural Information Processing Systems},
year={2025},
url={https://openreview.net/forum?id=V4YAS7NLXi}
}

\appendix

\section{Discrete subgroups of $\mathrm{O}(2)$}
Since our principal focus is vision transformers equivariant to discrete
subgroups of $\mathrm{O}(2)$, it is natural to first discuss the properties of
these subgroups. Due to the compactness of $\mathrm{O}(2)$, any discrete
subgroup is finite, and is isomorphic to either a cyclic group $C_n$ or a
dihedral group $D_n$, with the former generated by rotation by $2\pi/n$, and
the latter by rotation by $2\pi/n$ and a reflection.

In the following, we will fix the notation used to denote elements of $C_n$ and
$D_n$  and summarize the basic results on the representation theory over
$\mathbb{R}$ of these groups.
In general, we adopt the point of view that the groups $C_n$ and $D_n$ exist 
as abstract groups themselves, independently of 
the natural injective homomorphisms $C_n \hookrightarrow \mathrm{O}(2)$, $D_n \hookrightarrow \mathrm{O}(2)$.

We will always use $e$ to denote the identity element of a group.
For notational simplicity, we will write
\begin{equation}
\mathbf{R}(\theta) := 
\begin{pmatrix}
    \cos \theta & -\sin\theta\\
    \sin \theta & \cos\theta
    \end{pmatrix}
\end{equation}
for the rotation matrix by $\theta$.

\subsection{Cyclic Groups}
For a positive integer $n$, the cyclic group $C_n$ is generated 
by the symbol $r$, subject to the relation $r^n = e$. It is of order (cardinality) $n$, and is isomorphic to $\mathbb{Z}/n\mathbb{Z}$.

\begin{proposition}\leavevmode
\begin{enumerate}[label=(\roman*)]
\item If $n$ is even, the group $C_n$ has $\frac{n}{2} + 1$ real irreps, labeled by 
$\rmA, \rmB, \rmE_1, \rmE_2, \dotsb, \rmE_{\frac{n}{2}-1}$, where
\begin{equation}
\label{eq:cyclic_rep_matrices}
    \rho_{\rmA}(r) = \begin{pmatrix}
    1
    \end{pmatrix} \quad
    \rho_{\rmB}(r) = \begin{pmatrix}
    -1
    \end{pmatrix} \quad
    \rho_{\rmE_{k}}(r) = \mathbf{R}\left(\frac{2\pi k}{n}\right).
\end{equation}
\item If $n$ is odd, the group $C_n$ has $\frac{n+1}{2}$ real irreps, labeled by
$\rmA, \rmE_1, \rmE_2, \dotsb, \rmE_{\frac{n-1}{2}}$,
and the representation matrices are the same as the ones given in Eq.~\eqref{eq:cyclic_rep_matrices}.
\item A real irrep of a cyclic group is of real type if it is one-dimensional ($\rmA_k$ or $\rmB_k$), otherwise it is of complex type ($\rmE_k$).
\end{enumerate}
\end{proposition}

\subsection{Dihedral Groups}
The dihedral group $D_n$ is generated by two symbols, 
$t$ and $r$, subject to the relations $r^n=e$ and $tr = r^{-1}t$.

\begin{proposition}\leavevmode
\begin{enumerate}[label=(\roman*)]
\item If $n$ is even, the group $D_n$ has $\frac{n}{2} + 3$ real irreps, labeled by 
$\rmA_1, \rmA_2, \rmB_1, \rmB_2, \rmE_1, \rmE_2, \dotsb, \rmE_{\frac{n}{2}-1}$, where
\begin{equation}
\label{eq:dihedral_rep_matrices}
\begin{aligned}
    \rho_{\rmA_k}(r) = \begin{pmatrix}1\end{pmatrix} \quad
    \rho_{\rmA_k}(t) = \begin{pmatrix}(-1)^{k+1}\end{pmatrix} &\quad
    \rho_{\rmB_k}(r) = \begin{pmatrix}-1 \end{pmatrix} \quad
    \rho_{\rmB_k}(t) = \begin{pmatrix}(-1)^{k+1} \end{pmatrix} \quad\\
    \rho_{\rmE_{k}}(r) = \mathbf{R}\left(\frac{2\pi k}{n}\right)&\quad
    \rho_{\rmE_k}(t) = \begin{pmatrix} 1 & 0 \\ 0 & -1\end{pmatrix}.
\end{aligned}
\end{equation}
\item If $n$ is odd, the group $C_n$ has $\frac{n+3}{2}$ real irreps, labeled by
$\rmA_1, \rmA_2, \rmE_1, \rmE_2, \dotsb, \rmE_{\frac{n-1}{2}}$,
and the representation matrices are the same as the ones given in Eq.~\eqref{eq:dihedral_rep_matrices}.
\item All real irreps of $D_n$ are of real type for any $n$.
\end{enumerate}
\end{proposition}

\section{Irrep Multiplicities of Homogeneous Spaces of $D_4$ and $D_6$}
\label{app:irrep-mults}

Here, we provide the irrep multiplicities of $C(X_\alpha, \mathbb{R})$ for
all homogeneous spaces $X_\alpha$ of $G=D_4$ and $D_6$,
which are used in the MLP layer of each transformer block.
For each homogeneous space, we indicate the stabilizer subgroup $H_\alpha$
as well as the number of elements $|X_\alpha|$.

\begin{table}[H]
	\centering
	\begin{tabular}{|c|c|c|c|c|c|c|}
		\hline
		$H_\alpha$         & $\Gamma^\alpha_{\rmA_1}$ & $\Gamma^\alpha_{\rmA_2}$ & $\Gamma^\alpha_{\rmB_1}$ & $\Gamma^\alpha_{\rmB_2}$ & $\Gamma^\alpha_{\rmE_1}$ & $|X_\alpha|$ \\
		\hline
		$\{e\}$            & 1                         & 1                         & 1                         & 1                         & 2                        & 8            \\
		\hline
		$\braket{r^2}$     & 1                         & 1                         & 1                         & 1                         & 0                        & 4            \\
		\hline
		$\braket{r}$       & 1                         & 1                         & 0                         & 0                         & 0                        & 2            \\
		\hline
		$\braket{t}$       & 1                         & 0                         & 1                         & 0                         & 1                        & 4            \\
		\hline
		$\braket{tr}$      & 1                         & 0                         & 0                         & 1                         & 1                        & 4            \\
		\hline
		$\braket{t,r^2}$   & 1                         & 0                         & 1                         & 0                         & 0                        & 2            \\
		\hline
		$\braket{tr,r^2}$  & 1                         & 0                         & 0                         & 1                         & 0                        & 2            \\
		\hline
		$\braket{t,r}=D_4$ & 1                         & 0                         & 0                         & 0                         & 0                        & 1            \\
		\hline
	\end{tabular}
	\caption{Homogeneous spaces of $D_4$}
\end{table}

\begin{table}[H]
	\centering
	\begin{tabular}{|c|c|c|c|c|c|c|c|}
		\hline
		$H_\alpha$         & $\Gamma^\alpha_{\rmA_1}$ & $\Gamma^\alpha_{\rmA_2}$ & $\Gamma^\alpha_{\rmB_1}$ & $\Gamma^\alpha_{\rmB_2}$ & $\Gamma^\alpha_{\rmE_1}$ & $\Gamma^\alpha_{\rmE_2}$ & $|X_\alpha|$ \\
		\hline
		$\{e\}$            & 1                         & 1                         & 1                         & 1                         & 2                         & 2                         & 12           \\
		\hline
		$\braket{r^3}$     & 1                         & 1                         & 0                         & 0                         & 0                         & 2                         & 6            \\
		\hline
		$\braket{r^2}$     & 1                         & 1                         & 1                         & 1                         & 0                         & 0                         & 4            \\
		\hline
		$\braket{r}$       & 1                         & 1                         & 0                         & 0                         & 0                         & 0                         & 2            \\
		\hline
		$\braket{t}$       & 1                         & 0                         & 1                         & 0                         & 1                         & 1                         & 6            \\
		\hline
		$\braket{tr}$      & 1                         & 0                         & 0                         & 1                         & 1                         & 1                         & 6            \\
		\hline
		$\braket{t, r^3}$  & 1                         & 0                         & 0                         & 0                         & 0                         & 1                         & 3            \\
		\hline
		$\braket{t, r^2}$  & 1                         & 0                         & 1                         & 0                         & 0                         & 0                         & 2            \\
		\hline
		$\braket{tr, r^2}$ & 1                         & 0                         & 0                         & 1                         & 0                         & 0                         & 2            \\
		\hline
		$\braket{t,r}=D_6$ & 1                         & 0                         & 0                         & 0                         & 0                         & 0                         & 1            \\
		\hline
	\end{tabular}
	\caption{Homogeneous spaces of $D_6$}
\end{table}

\section{Proof of Lemma 1}
Let $\sigma: \mathbb{R}\rightarrow \mathbb{R}$ be some function satisfying
$\sigma(0) = 0$. We will use the same symbol to denote the entrywise application
$\mathbb{R}^n\rightarrow \mathbb{R}^n$ of $\sigma$. Let $(e_1, e_2, \dotsb,
e_n)$ be the standard basis for $\mathbb{R}^n$, with respect to which we define the representation matrices $D(g)$:
\begin{equation}
ge_i = \sum_{j=1}^n D_{ji}(g) e_j
\end{equation}
Equivariance at $e_i\in \mathbb{R}^n$ reads
\begin{equation}
    g \sigma(e_i) =  \sigma(ge_i).
\end{equation}
That is,
\begin{equation}
 \sigma(1)   \sum_{j=1}^n D_{ji}(g)  e_j = 
 \sigma\left(\sum_{j=1}^n D_{ji}(g) e_j\right)
 = \sum_{j=1}^n\sigma\left(D_{ji}(g)\right) e_j.
\end{equation}
Comparing the coefficient of $e_j$ gives
\begin{equation}
\label{eq:entrywise-equiv}
\sigma(1) D_{ji}(g) = \sigma(D_{ji}(g)).
\end{equation}
If $D_{ji}(g)\notin\{0,1\}$, then we can find a function $\sigma$ with
$\sigma(0)=0$ that violates Eq.~\eqref{eq:entrywise-equiv}.
Hence, it must be that $D_{ji}(g) \in \{0,1\}$. Note that this holds for any $g\in G$.

Now, consider the equation
\begin{equation}
    D(g) D(g^{-1}) = \mathbbm{1}_n,
\end{equation}
where $\mathbbm{1}_n$ is the $n\times n$ identity matrix.
We will now show by contradiction that both $D(g)$ and $D(g^{-1})$ are permutation matrices.
Suppose the $j$-th row of $D(g)$ has at least two $1$'s. 
Without loss of generality, we may assume
\begin{equation}
D(g) = \begin{pmatrix}
1 & 1 & \star & \dotsb & \star \\
\star & \star & \star  & \dotsb & \star \\
\vdots & \vdots & \vdots & \ddots & \vdots \\
\star & \star & \star  & \dotsb & \star \\
\end{pmatrix},
\end{equation}
where each $\star$ can be either $0$ or $1$. Then, the first two entries of
every column of $D(g^{-1})$ except for the first one must be zero, since otherwise $D(g)D(g^{-1})=\mathbbm{1}$ would not hold. That is,
\begin{equation}
    D(g^{-1})
    = \begin{pmatrix}
        a_1 & 0 & \dotsb & 0 \\
        a_2 & 0 & \dotsb & 0 \\
        a_2 & \star & \dotsb & \star \\
        \vdots & \vdots & \ddots & \vdots \\
        a_n & \star & \dotsb & \star \\
    \end{pmatrix}.
\end{equation}
Since $D(g^{-1})$ is invertible, we have $a_1=a_2 = 1$ (otherwise $D(g^{-1})$ would not have full rank).
But then $(D(g) D(g^{-1}))_{11} \ge 2$, contradicting $D(g) D(g^{-1}) = \mathbbm{1}_n$.

\qed

\section{Proof of Theorem 1}
The claim is trivial for $L=1$, so we assume $L > 1$.

If $w_v=0$, then $\mathrm{attn}$ sends all token sequences to $(b_v, \dotsb b_v)$.
In this case, $\mathrm{attn}$ is equivariant if and only if $b_v\in V$ is invariant,
and the claim follows easily.
In the following, we assume $w_v \neq 0$. 

Let $\phi_q, \phi_k, \phi_v: V\rightarrow V$ be the query, key, and value maps, which are affine by assumption. We will write $\phi_q(z) = w_q z + \beta_q$ etc.

Let $x = (x_1, \dotsb, x_L) \in \mathbb{R}^L \otimes V$ be a token sequence.
The attention layer acts according to
\begin{equation}
\begin{aligned}
    &\mathrm{attn}(x)_a 
    = \frac{\sum_{b=1}^Le^{\braket{\phi_q(x_a), \phi_k(x_b)}} \phi_v(x_b)}{\sum_{b=1}^Le^{\braket{\phi_q(x_a), \phi_k(x_b)}}}\\
        &= \frac{\sum_{b=1}^Le^{\braket{w_qx_a, w_k x_b} + \braket{w_k^t\beta_q, x_b} + \braket{x_a, w_q^t\beta_k} + \braket{\beta_q, \beta_k}} \phi_v(x_b)}
        {e^{\braket{w_qx_a, w_k x_b} + \braket{w_k^t\beta_q, x_b} + \braket{x_a, w_q^t\beta_k} + \braket{\beta_q, \beta_k}}}\\
    &= \frac{\sum_{b=1}^Le^{\braket{x_a, M x_b} + \braket{w_k^t\beta_q, x_b}} \phi_v(x_b)}
        {e^{\braket{x_a, M x_b} + \braket{w_k^t\beta_q, x_b}}},\\
    \end{aligned}
\end{equation}
where $M := w_q^t w_k$. $G$-equivariance reads
\begin{equation}
\label{eq:attn-equiv-condition}
\frac{\sum_{b=1}^Le^{\braket{x_a, Mx_b} + \braket{w_k^t\beta_q, x_b}} g\phi_v(x_b)}{\sum_{b=1}^Le^{\braket{x_a, Mx_b} + \braket{w_k^t\beta_q, x_b}}}
= 
\frac{\sum_{b=1}^Le^{\braket{x_a, g^{-1}Mgx_b} + \braket{g^{-1}w_k^t\beta_q, x_b}} \phi_v(g x_b)}{\sum_{b=1}^Le^{\braket{x_a, g^{-1}Mg x_b} + \braket{g^{-1}w_k^t\beta_q, x_b}}}
\end{equation}
for all $g\in G$.
For any vector $z\in V$, the above equation for the token sequence $(z, z, \dotsb, z)$ reads
\begin{equation}
g \phi_v(z) = \phi_v(gz).
\end{equation}
That is, $\phi_v$ is $G$-equivariant. This in particular implies $g\beta_v = \beta_v$ for all $g\in G$.
Multiplying Eq.~\eqref{eq:attn-equiv-condition} by $g^{-1}$ and cancelling out the value bias, we get
\begin{equation}
\label{eq:attn-equiv-1}
\frac{\sum_{b=1}^Le^{\braket{x_a, Mx_b} + \braket{w_k^t\beta_q, x_b}} w_vx_b}{\sum_{b=1}^Le^{\braket{x_a, Mx_b} + \braket{w_k^t\beta_q, x_b}}}
= 
\frac{\sum_{b=1}^Le^{\braket{x_a, g^{-1}Mgx_b} + \braket{g^{-1}w_k\beta_q, x_b}} w_v x_b}{\sum_{b=1}^Le^{\braket{x_a, g^{-1}Mg x_b} + \braket{g^{-1}w_k\beta_q, x_b}}}.
\end{equation}
Let $z, z'\in V$ be arbitrary, and let $x$ be the token sequence for which $x_a= z, x_b=z'$ and $x_{b'}=0$ for all $b' \notin \{a,b\}$. Then
\begin{equation}
    \label{eq:attn-equiv-2}
    \begin{aligned}
    &\forall z,z'\in V, g\in G:\\
    &\frac
    {e^{\braket{z, Mz} + \braket{w_k^t\beta_q, z}}w_vz + e^{\braket{z, Mz'} + \braket{w_k^t\beta_q, z'}} w_vz'}
    {L-2 + e^{\braket{z, Mz} + \braket{w_k^t\beta_q, z}} + e^{\braket{z, Mz'} + \braket{w_k^t\beta_q, z'}}}\\
    &= 
    \frac
    {e^{\braket{z, g^{-1}Mgz} + \braket{g^{-1}w_k^t\beta_q, z}}w_vz + e^{\braket{z, g^{-1}Mgz'} + \braket{g^{-1}w_k^t\beta_q, z'}} w_vz'}
    {L-2 + e^{\braket{z, g^{-1}Mgz} + \braket{g^{-1}w_k^t\beta_q, z}} + e^{\braket{z, g^{-1}Mgz'} + \braket{g^{-1}w_k^t\beta_q, z'}}}.
    \end{aligned}
\end{equation}
Taking $z'=0$ in this equation gives
\begin{equation}
\label{eq:attn-equiv-3}
\forall z \in V\setminus \mathrm{ker}(w_v):  
    \frac
    {e^{\braket{z, Mz} + \braket{w_k^t\beta_q, z}}}
    {L-1 + e^{\braket{z, Mz} + \braket{w_k^t\beta_q, z}}}
    =
    \frac
    {e^{\braket{z, g^{-1}Mgz} + \braket{g^{-1}w_k^t\beta_q, z}}}
    {L-1 + e^{\braket{z, g^{-1}Mgz} + \braket{g^{-1}w_k^t\beta_q, z}}}.
\end{equation}
The set $V \setminus \mathrm{ker}(w_v)$ is open and nonempty (because $w_v\neq0$). 
Since Eq.~\eqref{eq:attn-equiv-3} is analytic in $z$ and holds on an nonempty open set, it holds for all $z\in V$. By the monotonicity of the function $u\mapsto \frac{e^u}{L-1+e^u}$, we conclude 
\begin{equation}
\forall z\in V, g\in G: \braket{z, Mz} + \braket{w_k^t\beta_q, z}
= \braket{z, g^{-1}Mg z} + \braket{g^{-1}w_k^t\beta_q, z}.
\end{equation}
Now, go back to Eq.~\eqref{eq:attn-equiv-2} and again use analyticity and monotonicity 
to arrive at
\begin{equation}
    \forall z,z'\in V, g\in G:
    \braket{z, Mz'} + \braket{w_k^t\beta_q, z'} = \braket{z, g^{-1}Mgz'} + \braket{g^{-1}w_k^t\beta_q, z'}.
\end{equation}
Setting $z=0$ gives $w_k^t\beta_q = g^{-1} w_k^t\beta_q$, which in turn implies 
$M = g^{-1} M g$. We can now take $\tilde w_k = \mathbbm{1}$, $\tilde w_q = M$,
and $\tilde \beta_q  = w_k^t\beta_q$, so that $\tilde w_q, \tilde w_k$ are $G$-equivariant, $\tilde \beta_q$ is $G$-invariant, and $(\tilde w_q, \tilde \beta_q, \tilde w_k, 0, w_v, \beta_v)$ yields the same self-attention layer.


%

\qed

\section{Proof of Theorem 2}
In principle, the theorem could be proved in an abstract and almost trivial way,
more or less by noting that any $G$-representation is also an $H$-representation
and any $G$-equivariant map is also $H$-equivariant.
Here, we choose to provide a much more verbose proof, because it has the additional
advantage of giving a recipe to map a $G$-equivariant ViT to an $H$-equivariant one.

\textit{Step 1: patch embedding and positional encoding.}
Let $l\in \mathrm{Hom}_G(\mathbb{R}^U\otimes \mathbb{R}^3, V)$ denote the
operation of patch embedding followed by the addition of positional encodings.
Lemma~\ref{lem:restriction} applied to $\mathbb{R}^{U}\otimes \mathbb{R}^3$, $V$, and $W$
with $j=\mathrm{Id}_{\mathbb{R}^U\otimes \mathbb{R}^3}$ 
implies that there exists an $\tilde l \in \mathrm{Hom}_H(\mathbb{R}^U\otimes \mathbb{R}^3, W)$
such that the following diagram commutes:
\begin{equation}
\begin{tikzcd}
& &  V  \arrow{dd}{\mathrm{Res}^G_H}\\
    & \mathbb{R}^U \otimes \mathbb{R}^3 
    \arrow{ur}{l} 
    \arrow[swap]{dr}{\tilde l}\\
&& W
\end{tikzcd}
\end{equation}
The map $\tilde l$ is then expressible in terms of an $H$-equivariant patch
embedding layer.

Let $p \in \mathbb{R}^{\mathcal{H}}\otimes V$ be the positional encodings for 
the $G$-equivariant model. Then $(\mathrm{Id}\otimes \mathrm{Res}^G_H)(p)\in \mathbb{R}^{\mathcal{H}}\otimes W$ is $H$-invariant.

\textit{Step 2: multi-head self-attention.}

First, note that the number of heads $h$ divides each irrep multiplicity
$D_\sigma$ (because it is a linear combination of the $C_\rho$ with integer coefficients).

We will show that the $H$-equivariant model can output the same attention
scores. Let $\phi: V \rightarrow \mathbb{R}^h \otimes V_1$, with $V_1 =
\bigoplus_{\rho\in \widehat{G}} \mathbb{R}^{C_\rho/h}\otimes V_\rho$ denoting the
``reshaping'' linear isomorphism, which is $G$-equivariant. Similarly, we have
an $H$-equivariant linear isomorphism $\psi: W\rightarrow \mathbb{R}^h\otimes
W_1$. Note that $V_1$ and $W_1$ are isomorphic as $H$-representations, so we can find
an $H$-equivariant isometric linear isomorphism $\Phi: V_1 \rightarrow W_1$.
Consider the following diagram:
\begin{equation}
\begin{tikzcd}
V \arrow{r}\arrow[swap]{d}{\mathrm{Res}^G_H} \arrow{r}{x \mapsto q\oplus k} &\mathbb{R}^2\otimes V \arrow{r}{\phi}& \mathbb{R}^2\otimes \mathbb{R}^h \otimes V_1 \arrow{d}{\mathrm{Id}\otimes \mathrm{Id} \otimes \Phi}\\
W \arrow[swap]{r}{\color{red}\tilde f_{kq}} & \mathbb{R}^2\otimes W \arrow[swap]{r}{\psi} & \mathbb{R}^2\otimes\mathbb{R}^h \otimes W_1.
\end{tikzcd}
\end{equation}
By Lemma~\ref{lem:restriction}, there exists an affine $H$-equivariant map $\tilde f: W \rightarrow \mathbb{R}^2\otimes \mathbb{R}^h\otimes W_1$ making the diagram commute. 
Since $\psi$ is an equivariant linear isomorphism, there is an affine map $\tilde f_{kq}: W\rightarrow \mathbb{R}^2\otimes W$ such that this diagram commutes.
The map $\tilde f_{kq}$ then computes the key and query for the $H$-equivariant model.

Since $\Phi$ is an isometry, it is clear that the following diagram commutes:
\begin{equation}
\begin{tikzcd}
\mathbb{R}^{\mathcal{H}}\otimes \mathbb{R}^2\otimes \mathbb{R}^h\otimes V_1 
\arrow{dr}{\text{compute raw attention scores}}
\arrow[swap]{dd}{\mathrm{Id}\otimes \mathrm{Id}\otimes \mathrm{Id}\otimes \Phi}
\\
&\mathbb{R}^{\mathcal{H}}\otimes \mathbb{R}^{\mathcal{H}} \otimes \mathbb{R}^h
\\
\mathbb{R}^{\mathcal{H}}\otimes \mathbb{R}^2\otimes \mathbb{R}^h\otimes W_1 
\arrow[swap]{ur}{\text{compute raw attention scores}}
\end{tikzcd}
\end{equation}

Finally, consider the following diagram:
\begin{equation}
\begin{tikzcd}[row sep=1.em, column sep=3em]
    \mathbb{R}^{\mathcal{H}}\otimes \arrow{r}{\mathrm{Res}^G_H}V \arrow[swap]{d}{\mathrm{Id}\otimes (x\mapsto v)}& \mathbb{R}^{\mathcal{H}}\otimes W \arrow{d}{\color{red}\mathrm{Id}\otimes \tilde f_v}\\
    \mathbb{R}^{\mathcal{H}}\otimes V \arrow[swap]{d}{\mathrm{Id}\otimes \phi}& \mathbb{R}^{\mathcal{H}}\otimes W \arrow{d}{\mathrm{Id}\otimes \psi}\\
    \mathbb{R}^{\mathcal{H}}\otimes \mathbb{R}^h\otimes V_1 \arrow[swap]{d}{B\otimes \mathrm{Id}}
        \arrow{r}{\mathrm{Id}\otimes \mathrm{Id}\otimes \Phi}& \mathbb{R}^{\mathcal{H}}\otimes \mathbb{R}^h \otimes W_1 \arrow{d}{B\otimes \mathrm{Id}}\\
    \mathbb{R}^{\mathcal{H}}\otimes \mathbb{R}^h\otimes V_1 \arrow[swap]{d}{\mathrm{Id}\otimes \phi^{-1}}
    \arrow{r}{\mathrm{Id}\otimes \mathrm{Id}\otimes \Phi}
    & \mathbb{R}^{\mathcal{H}}\otimes \mathbb{R}^h \otimes W_1 \arrow{d}{\mathrm{Id}\otimes \psi^{-1}}\\
     \mathbb{R}^{\mathcal{H}}\otimes V \arrow[swap]{d}{\mathrm{Id}\otimes \mathbf{p}}& \mathbb{R}^{\mathcal{H}}\otimes W \arrow{d}{\color{red}\mathrm{Id}\otimes \mathbf{\tilde p}}\\
    \mathbb{R}^{\mathcal{H}}\otimes V \arrow{r}{\mathrm{Res}^G_H} & \mathbb{R}^{\mathcal{H}}\otimes W, \\
\end{tikzcd}
\end{equation}
where we applied Lemma~\ref{lem:restriction} twice to get the maps $\tilde f_v, \mathbf{\tilde p}\in \mathrm{Aff}_H(W, W)$, and $B\in \mathrm{End}(\mathbb{R}^{\mathcal{H}}\otimes\mathbb{R}^h)$
is the block-diagonal matrix that aggregates the value vectors according to the attention scores.
The middle block commutes because $B$ and $\Phi$ act on different factors of the tensor product.

\textit{Step 3: MLP.}

Let 
\begin{equation}
Y := \bigsqcup_{\beta\in \mathrm{Sub}(H)/\sim} \bigsqcup_{s=1}^{m_\beta}Y_\beta,
\end{equation}
where $Y_\beta$ is the $\beta$-th homogeneous space of $H$.
Clearly, $X\cong Y$ as $H$-sets.
Any isomorphism induces an isomorphism of $H$-representations $C(X, \mathbb{R})\xrightarrow{\sim} C(Y, \mathbb{R})$ 
that commutes with maps that acts entrywise. 
Thus, by Lemma~\ref{lem:restriction}, there exist $H$-equivariant affine maps $\tilde l_1, \tilde l_2$ such that the following diagram commutes:
\begin{equation}
\begin{tikzcd}
V \arrow[swap]{d}{\mathrm{Res}^G_H}\arrow{r}{l_1} & C(X, \mathbb{R})\arrow{r}{\sigma} \arrow{d}{\sim}& C(X, \mathbb{R}) \arrow{r}{l_2}\arrow{d}{\sim}  &  V \arrow{d}{\mathrm{Res}^G_H}\\
W \arrow[swap]{r}{\color{red}{\tilde l}_1}&  C(Y, \mathbb{R}) \arrow[swap]{r}{\sigma}&  C(Y, \mathbb{R})\arrow[swap]{r}{\color{red}\tilde l_2}&  W\\
\end{tikzcd}
\end{equation}

\textit{Step 4: residual connections.}
It remains to show that residual connections do not ruin any of the expressivity proofs above,
which is true since the following diagram commutes:
\begin{equation}
\begin{tikzcd}
    \mathbb{R}^{\mathcal{H}}\otimes V \arrow{r}{x\mapsto x+f(x)} \arrow[swap]{d}{\mathrm{Id}\otimes \mathrm{Res}^G_H} & \mathbb{R}^{\mathcal{H}}\otimes V \arrow{d}{\mathrm{Id}\otimes \mathrm{Res}^G_H}\\
    \mathbb{R}^{\mathcal{H}}\otimes W \arrow[swap]{r}{y \mapsto y + \tilde f(y)} & \mathbb{R}^{\mathcal{H}}\otimes W \\
\end{tikzcd}
\end{equation}
as long as the same diagram without residual connections also commutes.

\textit{Step 5: strictness of the inclusion.}
Assume $\dim \mathrm{Hom}_H(\mathbb{R}^U, W)> \dim\mathrm{End}_G(\mathbb{R}^U, V)$.
To show that the inclusion is strict, it suffices to find one function in
$\mathcal{F}_H[\delta, h; (D_\sigma)_{\sigma\in \widehat{G}}, (m_\beta)_{\beta\in
\mathrm{Sub}(H)/\sim}]$ that is not $G$-equivariant.
Consider an $H$-equivariant Vision transformer $\text{E-ViT}$ defined by setting
all parameters except the ones in the patch embedding layer to zero. Then,
$\mathrm{PosEnc}=\mathrm{Block}_k=\mathrm{Id}$ for all $k$, and $\text{E-ViT}
= \mathrm{PE}$. Since $\mathrm{Hom}_H(\mathbb{R}^U, W)$ is bigger than 
$\mathrm{Hom}_G(\mathbb{R}^U, V)$, it is possible to choose $\mathrm{PE}$ to
be $H$-equivariant but not $G$-equivariant.
\qed

\begin{lemma}
    \label{lem:restriction}
    Let $V, V'$ be $G$-representations and  
    suppose $W,W'$ are $H$-representations 
    such that there exist 
    isomorphisms $j\in \mathrm{Hom}_H(V, W)$ and $j'\in \mathrm{Hom}_H(V',W')$.
    Then for all $f\in \mathrm{Aff}_G(V, V')$, 
    there exists an $\tilde f \in \mathrm{Aff}_H(W, W')$
    such that the following diagram commutes:
    \begin{equation}
    \begin{tikzcd}
        V \arrow{r}{f} \arrow[swap]{d}{j} & V' \arrow{d}{j'}\\
        W \arrow[swap]{r}{\tilde f}& W'
    \end{tikzcd}
    \end{equation}
    \begin{proof}
    Write $f(x) = Ax + b$. Then $b$ is $G$-invariant and $A\in \mathrm{Hom}_G(V, V')$.
    We can take $\tilde f(y) = j'Aj^{-1}y + j'b$.
    \end{proof}
\end{lemma}

\section{Experimental Setup Details}
Unless otherwise stated, the following hyperparameters are always chosen in our
experiments (Section 5 of the main text):
\begin{table}[H]
\centering
\begin{tabular}{|c|c|}
\hline
    depth & 6\\
    \hline
    attention type & coupled\\
    \hline
    number of attention heads & 3\\
    \hline
    attention dropout & 0.1 \\
    \hline
    transformer block drop path & 0.05 \\
    \hline
    classification head dropout & 0.1\\
    \hline
    optimizer & AdamW \\
    \hline
    learning rate &  0.001 \\
    \hline
    betas & 0.9, 0.999 \\
    \hline
    weight decay & 0.05 \\
    \hline
    loss function & unweighted cross-entropy \\
    \hline
\end{tabular}
\vspace{.5em}
\caption{Default hyperparameters for all experiments.}
\end{table}

For models operating on square grids (equivariant to $D_4$ or its subgroups), we
choose the image size to be $256^2$ with patch size $16^2$. For models
operating on hexagonal patches (equivariant to $D_6$ or its subgroups), each
patch ($U$) is a hexagonal lattice restricted to a regular hexagon with $N_2=9$ grid
points on each side, and the patches themselves ($\mathcal{H}$) form a hexagonal lattice 
with $N_1=9$ grid points on each side (see Fig.~2 in the main text for the case $N_1=5, N_2=9$). 
This results in $217$ hexagonal pixels per patch and $217$ hexagonal patches,
with a total of $42,073$ pixels in the whole image (note that the patches have nonempty overlap).

Each image is preprocessed by normalizing the input RGB values of each pixel to $[-1,1]$. No
data augmentation is performed. For the $D_6$ family, we perform bilinear
interpolation to convert a square image to one defined on $\mathcal{H}_0$ (union
of hexagonal patches).

For the first experiment (Section 5.1 of the main text),
the maximum number of epochs and early stopping patience are 600/60
for 10\% sample ratio, 150/50 for  40\% sample ratio, and 50/10 for 100\% sample ratio.

For the second experiment (Section 5.2 of the main text), 
we perform at least three runs for each combination of 
(sample ratio, attention type, feature dimension).
The maximum number of epochs and early stopping patience are 600/60
for 10\% sample ratio and 200/30 for 100\% sample ratio.

For the third experiment (Section 5.3 of the main text),
The maximum number of epochs and early stopping patience are 200/30.

\end{document}